\documentclass[a4paper]{article}
\usepackage[utf8x]{inputenc}

\usepackage[a4paper,top=3cm,bottom=2cm,left=3cm,right=3cm,marginparwidth=1.75cm]{geometry}

\usepackage{natbib}
\usepackage{amsmath,amsfonts,amsthm}
\usepackage{xspace}
\usepackage{algorithm}
\usepackage{algorithmic}
\usepackage{graphicx}
\usepackage[colorlinks=true, allcolors=blue]{hyperref}
\usepackage{booktabs}
\usepackage{multirow}
\usepackage{subcaption}
\usepackage{pgfplots}
\usepackage{filecontents}
\usepackage{tikz}
\usetikzlibrary{calc}
\usetikzlibrary{mindmap}
\usetikzlibrary{arrows}
\usetikzlibrary{shadows}
\usetikzlibrary{plotmarks}
\usetikzlibrary{backgrounds}
\usetikzlibrary{shapes}
\usetikzlibrary{shapes.symbols}

\newcommand{\twomax}{\textsc{Twomax}\xspace}
\newcommand{\onemax}{\textsc{Onemax}\xspace}
\newcommand{\zeromax}{\textsc{Zeromax}\xspace}

\newcommand{\harm}{\mathrm{H_n}}
\newcommand{\hamm}{\mathrm{H}}
\newcommand{\distw}[1]{\hamm(#1, \winner)}
\newcommand{\mindist}{d}

\newcommand{\polylog}[1]{\mathrm{polylog}(#1)}
\newcommand{\poly}[1]{\mathrm{poly}(#1)}
\newcommand{\Hamm}[2]{\hamm\mathord{\left(#1,#2\right)}}
\newcommand{\g}{\mathrm{G}}
\newcommand{\G}[2]{\g\mathord{\left(#1,#2\right)}}
\newcommand{\cp}{\mathrm{cp}}
\newcommand{\CP}[1]{\cp\mathord{\left(#1\right)}}
\newcommand{\B}{\mathcal{B}}
\newcommand{\ones}[1]{\left|#1\right|_1}

\newcommand{\winner}{x^{*}}

\newcommand{\di}{\mathrm{d}}
\newcommand{\dist}[2]{\di\mathord{\left(#1,#2\right)}}

\DeclareMathOperator*{\argmin}{arg\,min}

\newcommand{\muea}{{\upshape{}($\mu$+1)~EA}\xspace}
\newcommand{\mupluslambdaea}{{\upshape{}($\mu$+$\lambda$)~EA}\xspace}

\newcommand{\e}{\mathrm{E}}
\newcommand{\E}[1]{\e\mathord{\left(#1\right)}}
\newcommand{\prob}{\mathrm{Prob}}
\newcommand{\Prob}[1]{\prob\mathord{\left(#1\right)}}

\newcommand{\ie}{i.\,e.,\xspace}

\newcommand{\clearing}{{\it clearing}\xspace}
\newcommand{\clearingradius}{{\it clearing radius}\xspace}
\newcommand{\nichecapacity}{{\it niche capacity}\xspace}

\newcommand{\IFTHENELSE}[3]{\STATE \algorithmicif\ #1\ \algorithmicthen\ #2\ \algorithmicelse\ #3\ \algorithmicendif}

\newcommand{\ignore}[1]{}

\newtheorem{theorem}{Theorem}[section]
\newtheorem{lemma}[theorem]{Lemma}
\newtheorem{corollary}[theorem]{Corollary}
\newtheorem{definition}[theorem]{Definition}

\pgfplotsset{compat=1.14}
\pgfplotsset{
	enlargelimits=true,
  	grid=major,
  	scale only axis,
}

\pgfplotstableread[col sep = tab]{average30.txt}\averagethirty
\pgfplotstableread[col sep = tab]{average30log.txt}\averagethirtylog
\pgfplotstableread[col sep = tab]{average100.txt}\averagehundred
\pgfplotstableread[col sep = tab]{average100log.txt}\averagehundredlog
\pgfplotstableread[col sep = tab]{hamming_1.txt}\hammingone
\pgfplotstableread[col sep = tab]{hamming_2.txt}\hammingtwo

\title{On the Runtime Analysis of the Clearing Diversity-Preserving Mechanism}
\author{Edgar Covantes Osuna and Dirk Sudholt\\\normalsize{}Department of Computer Science\\\normalsize{}University of Sheffield, United Kingdom}

\begin{document}
\maketitle
\begin{abstract}
Clearing is a niching method inspired by the principle of assigning the available resources among a niche to a single individual. The clearing procedure supplies these resources only to the best individual of each niche: the winner. So far, its analysis has been focused on experimental approaches that have shown that clearing is a powerful diversity-preserving mechanism. Using rigorous runtime analysis to explain how and why it is a powerful method, we prove that a mutation-based evolutionary algorithm with a large enough population size, and a phenotypic distance function always succeeds in optimising all functions of unitation for small niches in polynomial time, while a genotypic distance function requires exponential time. Finally, we prove that with phenotypic and genotypic distances clearing is able to find both optima for \twomax and several general classes of bimodal functions in polynomial expected time. We use empirical analysis to highlight some of the characteristics that makes it a useful mechanism and to support the theoretical results.
\end{abstract}

\section{Introduction}
\label{sec:intro}

Evolutionary Algorithms (EAs) with elitist selection are suitable to locate the optimum of unimodal functions as they converge to a single solution of the search space. This behaviour is also one of the major difficulties in a population-based EA, the premature convergence toward a suboptimal individual before the fitness landscape is explored properly. Real optimisation problems, however, often lead to multimodal domains and so require the identification of multiple optima, either local or global~\citep{Sareni1998,Singh2006}.

In multimodal optimisation problems, there exist many attractors for which finding a global optimum can become a challenge to any optimisation algorithm. A diverse population can deal with multimodal functions and can explore several hills in the fitness landscape simultaneously, so they can therefore support global exploration and help to locate several local and global optima. The algorithm can offer several good solutions to the user, a feature desirable in multiobjective optimisation. Also, it provides higher chances to find dissimilar individuals and to create good offspring with the possibility of enhancing the performance of other procedures such as crossover~\citep{Friedrich2009}.

Diversity-preserving mechanisms provide the ability to visit many and/or different unexplored regions of the search space and generate solutions that differ in various significant ways from those seen before~\citep{Gendreau2010,Lozano2010}. Most analyses and comparisons made between diversity-preserving mechanisms are assessed by means of empirical investigations~\citep{Chaiyaratana2007,Ursem2002} or theoretical runtime analyses~\citep{Jansen2005,Friedrich2007,Oliveto2014,Oliveto2014a,Gao2014,Doerr2016}. Both approaches are important to understand how these mechanisms impact the EA runtime and if they enhance the search for obtaining good individuals. These different results imply where/which diversity-preserving mechanism should be used and, perhaps even more importantly, where they should not be used.

One particular way for diversity maintenance are the niching methods, such methods are based on the mechanics of natural ecosystems. A niche can be viewed as a subspace in the environment that can support different types of life. A specie is defined as a group of individuals with similar features capable of interbreeding among themselves but that are unable to breed with individuals outside their group. Species can be defined as similar individuals of a specific niche in terms of similarity metrics. In EAs the term niche is used for the search space domain, and species for the set of individuals with similar characteristics. By analogy, niching methods tend to achieve a natural emergence of niches and species in the search space~\citep{Sareni1998}.

A niching method must be able to form and maintain multiple, diverse, final solutions for an exponential to infinite time period with respect to population size, whether these solutions are of identical fitness or of varying fitness. Such requirement is due to the necessity to distinguish cases where the solutions found represent a new niche or a niche localised earlier~\citep{Mahfoud1995}.

Niching methods have been developed to reduce the effect of genetic drift resulting from the selection operator in standard EAs. They maintain population diversity and permit the EA to investigate many peaks in parallel. On the other hand, they prevent the EA from being trapped in local optima of the search space~\citep{Sareni1998}. In the majority of algorithms, this effect is attained due to the modification of the process of selection of individuals, which takes into account not only the value of the fitness function but also the distribution of individuals in the space of genotypes or phenotypes~\citep{Glibovets2013}.

Many researchers have suggested methodologies for introducing niche-preserving techniques so that, for each optimum solution, a niche gets formed in the population of an EA. Most of the analyses and comparisons made between niching mechanisms are assessed by means of empirical investigations using benchmark functions~\citep{Sareni1998,Singh2006}. There are examples where empirical investigations are used to support theoretical runtime analyses and close the gap between theory and practice~\citep{Friedrich2009,Oliveto2014a,Oliveto2015,Covantes2017,Covantes2017b}. 

Both fields use artificially designed functions to highlight characteristics of the studied EAs when tackling optimisation problems. They exhibit such properties in a very precise, distinct, and paradigmatic way. Moreover, they can help to develop new ideas for the design of new variants of EAs and other search heuristics. This leads to valuable insights about new algorithms on a solid basis~\citep{Jansen2013}.

Most of the theoretical analyses are made on example functions with a clear and concrete structure so that they are easy to understand. They are defined in a formal way and allow for the derivation of theorems and proofs allowing to develop knowledge about EAs in a sound scientific way. In the case of empirical analyses, most of the results are based on the analysis of more complex example functions and algorithmic frameworks for a specific set of experiments. This approach allows to explore the general characteristics more easily than in the theoretical field.

Our contribution is to provide a rigorous theoretical runtime analysis for the \clearing diversity mechanism, to find out whether the mechanism is able to provide good solutions by means of experiments, and we include theoretical runtime analysis to prove how and why an EA is able to obtain good solutions depending on how the population size, \clearingradius, \nichecapacity, and the dissimilarity measure are chosen.

This manuscript extends a preliminary conference paper~\citep{Covantes2017} in the following ways. We extend the theory for large niches by looking into the choice of the population size. While it was known that a population size of $\mu \ge \kappa n^2/4$ is sufficient~\citep{Covantes2017}, here we show that population sizes of at least $\Omega(n/\polylog{n})$ are necessary to escape from local optima. The reason is that for smaller population sizes, winners in local optima spawn offspring that repeatedly take over the whole population, and this happens before individuals can escape from the optima's basin of attraction. We further extend our analysis to more general classes of examples landscapes defined by~\citet{Jansen2016}, showing that \clearing is effective across a range of functions with different slopes and optima having different basins of attraction. Finally, we extend the experimental analysis for smaller population sizes of $\mu$ with small $n=30$ and large $n=100$. 

In the remainder of this paper, we first present the algorithmic framework in Section~\ref{sec:prelim}. The definition of \clearing, algorithmic approach, and dissimilarity measures are defined in Section~\ref{sec:clearing}. The theoretical analysis is divided in Sections~\ref{sec:small_niches} and~\ref{sec:large_niches} for small and large niches, respectively. In Section~\ref{sec:small_niches} we show how \clearing is able to solve, for small niches and the right distance function, all functions of unitation, and in Section~\ref{sec:large_niches} we show how the population dynamics with large population size in \clearing solves \twomax with the most natural distance function: Hamming distance while it fails with a small population size. In Section~\ref{sec:genexaland} we show that the analysis made in Section~\ref{sec:large_niches} is general enough to be applied into more general function classes. Section~\ref{sec:exp} contains the experimental results showing how well our theoretical results matches with empirical results for the general behaviour of the algorithm and providing a closer look into the impact of the population size on performance. We present our conclusions in Section \ref{sec:conclu}, giving additional insight into the dynamic behaviour of the algorithm.

\section{Preliminaries}
\label{sec:prelim}

We focus our analysis on the simplest EA with a finite population called~\muea (hereinafter, $\mu$ denotes the size of the current population, see Algorithm~\ref{alg:muea}). Our aim is to develop rigorous runtime bounds of the~\muea with the~\clearing diversity mechanism. We want to study how diversity helps to escape from some optima. The~\muea uses random parent selection and elitist selection for survival and has already been investigated by~\cite{Witt2006}.

\begin{algorithm}[!ht]
  \begin{algorithmic}[1]
  	\STATE Let $t:=0$ and initialise $P_0$ with $\mu$ individuals chosen uniformly at random.
    \WHILE{optimum \NOT found} 
    	\STATE{Choose $x \in P_t$ uniformly at random.} 
    	\STATE{Create $y$ by flipping each bit in $x$ independently with probability $1/n$.}
        \STATE{Choose $z \in P_t$ with worst fitness uniformly at random.}
        \IFTHENELSE{$f(y) \geq f(z)$}{$P_{t+1} = P_{t} \setminus \{z\}\cup \{y\}$}
        {$P_{t+1} = P_{t}$}
    	\STATE{Let $t:=t+1$.}
    \ENDWHILE
  \end{algorithmic}
  \caption{\muea}
  \label{alg:muea}
\end{algorithm}

We consider functions of unitation $f:\{0,1\}^n\rightarrow \mathbb{R}$, where $f(x)$ depends only on the number of 1-bits contained in a string $x$ and is always non-negative, \ie $f$ is entirely defined by a function $u:\{0,\ldots,n\}\rightarrow \mathbb{R}^{+}$, $f(x)=u(\ones{x})$, where $\ones{x}$ denotes the number of 1-bits in individual $x$. In particular, we consider the bimodal function of unitation called \twomax (see Definition~\ref{def:twomax}) for the analysis of large niches. \twomax can be seen as a bimodal equivalent of \onemax. The fitness landscape consists of two hills with symmetric slopes. In contrast to \cite{Friedrich2009} where an additional fitness value for $1^n$ was added to distinguish between a local optimum $0^n$ and a unique global optimum, we have opted to use the same approach as \cite{Oliveto2014a}, and leave unchanged \twomax since we aim at analysing the global exploration capabilities of a population-based EA.

\begin{definition}[\twomax]
\label{def:twomax}
A bimodal function of unitation which consists of two different symmetric slopes \zeromax and \onemax with $0^n$ and $1^n$ as global optima, respectively.
\[
\twomax(x):=\max\left\{\sum_{i=1}^{n}x_i,n-\sum_{i=1}^{n}x_i\right\}.
\]

The fitness of ones increases with more than~$n/2$ 1-bits, and the fitness of zeroes increases with less than~$n/2$ 1-bits. These sets are refereed as branches. The aim is to find a population containing both optima (see Figure~\ref{fig:twomax}).
\end{definition}

\begin{figure}[!ht]
	\centering
    \resizebox{.45\linewidth}{!}{
		\begin{tikzpicture}[domain=0:30,xscale=0.15,yscale=0.25,scale=1, every shadow/.style={shadow 
        xshift=0.0mm, shadow yshift=0.4mm}]
          \tikzstyle{helpline}=[black,thick];
          \tikzstyle{function}=[blue,very thick];
          \tikzstyle{optimum}=[red,very thick];
          \draw[gray!20,line width=0.2pt,xstep=1,ystep=1] (0,0) grid (30.5,15.5);
          \draw[function] (0,15) -- (15,0) -- (30,15);
          \draw[helpline, thin, -triangle 45] (0,0) -- (0,17);
          \draw[helpline, thin, -triangle 45] (0,0) -- (32,0) node[right] {\scriptsize \#{}ones};
          \draw[helpline] (0,0) -- (0,16);
          \draw[helpline] (0,0) -- (31,0);
          \draw[helpline] (0,0) -- (-1,0) node[left] {\scriptsize $n/2$};
          \draw[helpline] (0,0) -- (0,-1) node[below] {\scriptsize $0$};
          \draw[helpline] (0,15) -- (-1,15) node[left] {\scriptsize $n$};
          \draw[helpline] (15,0) -- (15,-1) node[below] {\scriptsize $n/2$};
          \draw[helpline] (30,0) -- (30,-1) node[below] {\scriptsize $n$};
          \filldraw[optimum] (0,15) circle (6pt);
          \filldraw[optimum] (30,15) circle (6pt);
		\end{tikzpicture}
        }
	\caption{Sketch of the function \twomax with $n=30$.}
	\label{fig:twomax}
\end{figure}
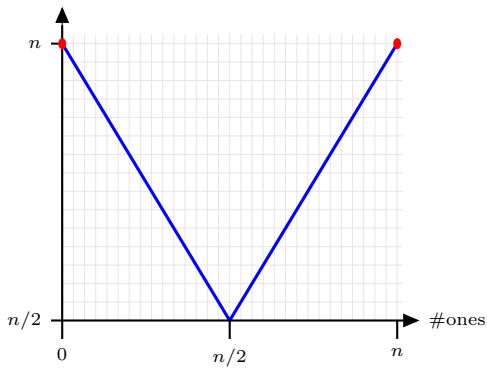

We analyse the expected time until both optima have been reached. \twomax is an ideal benchmark function for \clearing as it is simply structured, hence facilitating a theoretical analysis, and it is hard for EAs to find both optima as they have the maximum possible Hamming distance. Its choice further allows comparisons with previous approaches like in~\cite{Friedrich2009} and \cite{Oliveto2014a} in the context of diversity-preserving mechanisms. 

The \muea with no diversity-preserving mechanism (Algorithm~\ref{alg:muea}) has already been analysed for the \twomax function. The selection pressure is quite low, nevertheless, the \muea is not able to maintain individuals on both branches for a long time. Without any diversification, the whole population of the~\muea collapses into the $0^n$ branch with probability at least~$1/2-o(1)$ (see \citealp{Motwani1995} for the asymptotic notation) in time $n^{n-1}$, once the population contains copies of optimum individuals in one of the two branches, it will be necessary to flip all the bits at the same time in order to allow that individual to survive and find both optima on \twomax, so the expected optimisation time is~$\Omega(n^n)$ \citep[Theorem~1]{Friedrich2009}.

Adding other diversity-preserving mechanisms into the \muea like \emph{avoiding genotype or phenotype duplicates} does not work, the algorithm cannot maintain individuals in both branches, so the population collapse into the $0^n$ branch with probability at least $1/2-o(1)$ with an expected optimisation time of $\Omega(n^{n-1})$ and $2^{\Omega(n)}$ \citep[Theorem 2 and 3]{Friedrich2009}, respectively. \emph{Deterministic crowding} with sufficiently large population is able to reach both optima with high probability in expected time $O(\mu n \log n)$ \citep[Theorem 4]{Friedrich2009}. 

A \emph{modified version of fitness sharing} is analysed in \cite{Friedrich2009}: rather than selecting individuals based on their shared fitness, selection was done on a level of populations. The goal was to select the new population out of the union of all parents and all offspring such that it maximises the overall shared fitness of the population. The drawback of this approach is that all possible size-$\mu$ subsets of this union of size $\mu+\lambda$, where $\lambda$ is the number of offspring, need to be examined. For large $\mu$ and $\lambda$, this is prohibitive. It is proved that a population-based shared fitness approach with $\mu\geq 2$ reaches both optima of \twomax in expected time $O(\mu n\log{n})$ \citep[Theorem 5]{Friedrich2009}. 

In \cite{Oliveto2014a}, the performance of the original \emph{fitness sharing} approach is analysed. The analysis showed that using the conventional (phenotypic) sharing approach leads to considerably different behaviours. A population size $\mu=2$ is not sufficient to find both optima on \twomax in polynomial time: with probability $1/2+\Omega(1)$ the population will reach the same optimum, and from there the expected time to find both optima is $\Omega(n^{n/2})$ \citep[Theorem 1]{Oliveto2014a}. However, there is still a constant probability $\Omega(1)$ to find both optima in polynomial expected time $O(n \log{n})$, if the two search points are initialised on different branches, and if these two search points maintain similar fitness values throughout the run \citep[Theorem~2]{Oliveto2014a}. 

With $\mu\geq 3$, once the population is close enough to one optimum, individuals descending the branch heading towards the other optimum are accepted. This threshold, that allows successful runs with probability $1$, lies further away from the local optimum as the population size increases finding both optima in expected time $O(\mu n\log{n})$ \citep[Theorem 3]{Oliveto2014a}. Concerning the effects of the offspring population, increasing the offspring population of a \mupluslambdaea, with $\mu=2$ and $\lambda\geq \mu$ cannot guarantee convergence to populations with both optima, \ie depending on $\lambda$ one or both optima can get lost, the expected time for finding both optima is $\Omega(n^{n/2})$ \citep[Theorem 4]{Oliveto2014a}.

\section{Clearing}
\label{sec:clearing}
\emph{Clearing} is a niching method inspired by the principle of sharing limited resources within a niche (or subpopulation) of individuals characterised by some similarities. Instead of evenly sharing the available resources among the individuals of a niche, the \clearing procedure supplies these resources only to the best individual of each niche: the winner. The winner takes all rather than sharing resources with the other individuals of the same niche as it is done with fitness sharing~\citep{Petrowski1996}.

Like in fitness sharing, the \clearing algorithm uses a dissimilarity measure given by a threshold called \clearingradius $\sigma$ between individuals to determine if they belong to the same niche or not. The basic idea is to preserve the fitness of the individual that has the best fitness (also called dominant individual), while it resets the fitness of all the other individuals of the same niche to zero\footnote{We tacitly assume that all fitness values are larger than~0 for simplicity. In case of a fitness function~$f$ with negative fitness values we can change clearing to reset fitness to ${f_{\min}-1}$, where $f_{\min}$ is the minimum fitness value of~$f$, such that all reset individuals are worse than any other individuals.}. With such a mechanism, two approaches can be considered. For a given population, the set of winners is unique. The winner and all the individuals that it dominates are then fictitiously removed from the population. Then the algorithm proceeds in the same way with the new population which is then obtained. Thus, the list of all the winners is produced after a certain number of steps.

On the other hand, the population can be dominated by several winners. It is also possible to generalise the \clearing algorithm by accepting several winners chosen among the \nichecapacity $\kappa$ (best individuals of each niche defined as the maximum number of winners that a niche can accept). Thus, choosing niching capacities between one and the population size offers intermediate situations between the maximum \clearing ($\kappa = 1$) and a standard EA ($\kappa\geq\mu$).

Empirical investigations made in~\cite{Petrowski1996,Petrowski1997a,Petrowski1997b,Sareni1998,Singh2006} mentioned that \clearing surpasses all other niching methods because of its ability to produce a great quantity of new individuals by randomly recombining elements of different niches, controlling this production by resetting the fitness of the poor individuals in each different niche. Furthermore, an elitist strategy prevents the rejection of the best individuals.

We incorporate the \clearing method into Algorithm~\ref{alg:muea}, resulting in Algorithm \ref{alg:mueaclea}. The idea behind Algorithm~\ref{alg:mueaclea} is: once a population with $\mu$ individuals is generated, an individual $x$ is selected and changed according to mutation. A temporary population~$P_{t}^{*}$ is created from population $P_t$ and the offspring $y$, then the fitness of each individual in $P_{t}^{*}$ is updated according to the \clearing procedure shown in Algorithm~\ref{alg:clearing}.

\begin{algorithm}[!ht]
  \begin{algorithmic}[1]
  	\STATE Let $t:=0$ and initialise $P_0$ with $\mu$ individuals chosen uniformly at random.
    \WHILE{optimum \NOT found} 
    	\STATE{Choose $x \in P_t$ uniformly at random.} 
    	\STATE{Create $y$ by flipping each bit in $x$ independently with probability $1/n$.}
        \STATE{Let $P_t^* = P_t \cup \{y\}$.}
        \STATE{Update $f(P_t^*)$ with the clearing procedure (Algorithm \ref{alg:clearing}).}
        \STATE{Choose $z \in P_t$ with worst fitness uniformly at random.}
        \IFTHENELSE{$f(y) \geq f(z)$}{$P_{t+1} = P_{t}^{*} \setminus \{z\}$}
        {$P_{t+1} = P_{t}^{*} \setminus \{y\}$}
    	\STATE{Let $t:=t+1$.}
    \ENDWHILE
  \end{algorithmic}
  \caption{\muea with clearing}
  \label{alg:mueaclea}
\end{algorithm}

Each individual is compared with the winner(s) of each niche in order to check if it belongs to a certain niche or not, and to check if its a winner or if it is cleared. Here~$\dist{P[i]}{P[j]}$ is any dissimilarity measure (distance function) between two individuals $P[i]$ and $P[j]$ of population~$P$. Finally, we keep control of the \nichecapacity defined by~$\kappa$. For the sake of clarity, the replacement policy will be the one defined in \cite{Witt2006}: the individuals with best fitness are selected (set of winners) and individuals coming from the new generation are preferred if their fitness values are at least as good as the current ones (novelty is rewarded).

\begin{algorithm}[!ht]
  \begin{algorithmic}[1]
  	\STATE Sort $P$ according to fitness of individuals by decreasing values.
    \FOR{$i := 1$ \TO $|P|$} 
    	\IF{$f(P[i]) > 0$}
        	\STATE $winners := 1$
            \FOR{$j := i + 1$ \TO $|P|$} 
    			\IF{$f(P[j]) > 0$ \AND $\dist{P[i]}{P[j]} < \sigma$}
                	\IFTHENELSE{$winners < \kappa$}{$winners := winners + 1$}{$f(P[j]) := 0$}
                \ENDIF
    		\ENDFOR
        \ENDIF
    \ENDFOR
  \end{algorithmic}
  \caption{Clearing}
  \label{alg:clearing}
\end{algorithm}

Finally, as dissimilarity measures, we have considered genotypic or Hamming distance, defined as the number of bits that have different values in $x$ and $y$: $\dist{x}{y}:=\Hamm{x}{y}:=\sum_{i=0}^{n-1}|x_i - y_i|$, and phenotypic (usually defined as Euclidean distance between two phenotypes). As \twomax is a function of unitation, we have adopted the same approach as in previous work~\citep{Friedrich2009,Oliveto2014a} for the phenotypic distance function, allowing the distance function $\di$ to depend on the number of ones: $\dist{x}{y}:=|\ones{x}-\ones{y}|$ where $\ones{x}$ and $\ones{y}$ denote the number of 1-bits in individual $x$ and $y$, respectively.

\section{Small Niches}
\label{sec:small_niches}
In this section we prove that the \muea with phenotypic clearing and a small niche capacity is not only able to achieve both optima of \twomax but is also able to optimise all functions of unitation with a large enough population, while genotypic clearing fails in achieving such a task (hereinafter, we will refer as phenotypic or genotypic clearing to Algorithm~\ref{alg:clearing} with phenotypic or genotypic distance function, respectively). 

\subsection{Phenotypic Clearing}
\label{sec:phenotypic_clearing}
First it is necessary to define a very important property of \clearing, which is its capacity of preventing the rejection of the best individuals in the \muea, and once $\mu$ is defined large enough, \clearing and the population size pressure will always optimise any function of unitation.

Note that on functions of unitation all search points with the same number of ones have the same fitness, and for phenotypic clearing with \clearingradius $\sigma=1$ all search points with the same number of ones form a niche. We refer to the set of all search points with $i$ ones as niche~$i$. In order to find an optimum for any function of unitation, it is sufficient to have all niches~$i$, for $0 \le i \le n$, being present in the population.

In the \muea with phenotypic clearing with $\sigma=1$, $\kappa \in \mathbb{N}$ and $\mu \geq (n+1)\cdot\kappa$, a niche~$i$ can only contain $\kappa$ winners with $i$ ones. The condition on $\mu$ ensures that the population is large enough to store individuals from all possible niches. 

\begin{lemma}
\label{lem:largepop}
Consider the \muea with phenotypic clearing with $\sigma=1$, $\kappa \in \mathbb{N}$ and $\mu \geq (n+1)\cdot\kappa$ on any function of unitation. Then winners are never removed from the population, \ie if $x \in P_t$ is a winner then $x \in P_{t+1}$.
\end{lemma}
\begin{proof}
After the first evaluation with \clearing, individuals dominated by other individuals are cleared and the dominant individuals are declared as winners. Cleared individuals are removed from the population when new winners are created and occupy new niches. Once an individual becomes a winner, it can only be removed if the size of the population is not large enough to maintain it, as the worst winner is removed if a new winner reaches a new better niche. Since there are at most $n+1$ niches, each having at most~$\kappa$ winners, if $\mu\geq (n+1)\cdot\kappa$ then there must be a cleared individual amongst the $\mu+1$ parents and offspring considered for deletion at the end of the generation. Thus, a cleared individual will be deleted, so winners cannot be removed from the population. 
\end{proof}

The behaviour described above means, that with the defined parameters and sufficiently large~$\mu$ to occupy all the niches, we have enough conditions for the furthest individuals (individuals with the minimum and maximum number of ones in the population) to reach the opposite edges. Now that we know that a winner cannot be removed from the population by Lemma~\ref{lem:largepop}, it is just a matter of finding the expected time until $0^n$ and $1^n$ are found. 

Because of the elitist approach of the \muea, winners will never be replaced if we assume a large enough population size. In particular, the minimum (maximum) number of ones of any search point in the population will never increase (decrease). We first estimate the expected time until the two most extreme search points $0^n$ and $1^n$ are being found, using arguments similar to the well-known fitness-level method~\citep{Wegener2002}. 

\begin{lemma}
\label{lem:optunitation}
Let $f$ be a function of unitation and $\sigma=1$, $\kappa \in \mathbb{N}$ and $\mu \geq (n+1)\cdot\kappa$. Then, the expected time for finding the search points $0^n$ and $1^n$ with the \muea with phenotypic clearing on $f$ is $O(\mu n \log{n})$.
\end{lemma}
\begin{proof} 
First, we will focus on estimating the time until the $1^n$ individual is found (by symmetry, the same analysis applies for the $0^n$ individual). If the current maximum number of ones in any search point is~$i$, it has a probability of being selected at least of $1/\mu$. In order to create a niche  with $j>i$ ones, it is just necessary that one of the $n-i$ zeroes is flipped into 1-bit and the other bits remains unchanged. Each bit flip has a probability of being changed (mutated) of $1/n$ and the probability of the other bits remaining unchanged is $(1-1/n)^{n-1}$. Hence, the probability of creating some niche with $j>i$ ones is at least
\[
\frac{1}{\mu}\cdot\frac{n-i}{n} \cdot \left(1-\frac{1}{n}\right)^{n-1}\geq\frac{n-i}{\mu en}.
\]

The expected time for increasing the maximum number of ones, $i$, is hence at most $(\mu en)/(n-i)$ and the expected time for finding $1^n$ is at most
\[
\sum_{i=0}^{n-1}\frac{\mu en}{n-i}=\mu en\sum_{i=1}^{n}\frac{1}{i}\leq \mu en \ln{n}=O(\mu n \log{n}).
\]

Where the summation $\harm=\sum_{i=1}^{n}1/i$ is known as the \emph{harmonic number} and satisfies $\harm=\ln{n} + \Theta(1)$. Adding the same time for finding $0^n$ proves the claim.
\end{proof}

Once the search points $0^n$ and $1^n$ have been found, we can focus on the time required for the algorithm until all intermediate niches are discovered.

\begin{lemma}
\label{lem:fillemptyniches}
Let $f$ be any function of unitation, $\sigma=1$, $\kappa \in \mathbb{N}$ and $\mu \geq (n+1)\cdot\kappa$, and assume that the search points $0^n$ and $1^n$ are contained in the population. Then, the expected time until all niches are found with the \muea with phenotypic clearing on $f$ is $O(\mu n)$.
\end{lemma}
\begin{proof}
According to Lemma~\ref{lem:largepop} and the elitist approach of \muea, winners will never be replaced if we assume a large enough population size and by assumption we already have found both search points $0^n$ and $1^n$. 

As long as the algorithm has not yet found all niches with at least $n/2$ ones, then there must be an index $i\geq n/2$ such that the population does not cover the niche with $i$ ones, but it does cover the niche with $i+1$ ones. We can mutate an individual from niche $i+1$ to populate niche $i$. The probability of selecting an individual from niche $i+1$ is at least $1/\mu$, and since it is just necessary to flip one of at least $n/2$ 0-bits with probability $1/n$, we have a probability of at least $1/2$ to do so, and a probability of leaving the remaining bits untouched of $(1-1/n)^{n-1}\geq 1/e$. Together, the probability is bounded from below by $1/(2\mu e)$. Using the level-based argument used before for the case of the niches, the expected time to occupy all niches with at least $n/2$ ones is bounded by
\[
\sum_{i=n/2}^{n-1}\frac{2\mu e}{1}\leq 2\mu en = O(\mu n). 
\]
A symmetric argument applies for the niches with fewer than $n/2$ ones, leading to an additional time of $O(\mu n)$.
\end{proof}

\begin{theorem}
\label{the:optunitation}
Let $f$ be a function of unitation and ${\sigma=1}$, $\kappa \in \mathbb{N}$ and $\mu \geq (n+1)\cdot\kappa$. Then, the expected optimisation time of the \muea with phenotype \clearing on $f$ is $O(\mu n\log{n})$.
\end{theorem}
\begin{proof}
Now that we have defined and proved all conditions where the algorithm is able to maintain every winner in the population (Lemma~\ref{lem:largepop}), to find the extreme search points (Lemma~\ref{lem:optunitation}) and intermediate niches (Lemma~\ref{lem:fillemptyniches}) of the function $f$, we can conclude that the total time required to optimise the function of unitation $f$ is $O(\mu n\log{n})$.
\end{proof}

\subsection{Genotypic Clearing}
\label{sec:genotypic_clearing}

In the case of genotypic clearing with $\sigma=1$, the \muea behaves like the diversity-preserving mechanism called {\it no genotype duplicates}. The \muea with no genotype duplicates rejects the new offspring if the genotype is already contained in the population. The same happens for the \muea with genotypic clearing and $\sigma=1$ if the population is initialised with $\mu$ mutually different genotypes (which happens with probability at least $1-\binom{\mu}{2} 2^{-n}$). In other words, conditional on the population being initialised with mutually different search points, both algorithms are identical. In \citet[Theorem~2]{Friedrich2009}, it was proved that the \muea with no genotype duplicates and $\mu = o(n^{1/2})$ is not powerful enough to explore the landscape and can be easily trapped in one optimum of \twomax. Adapting~\citet[Theorem~2]{Friedrich2009} to the goal of finding both optima and noting that $\binom{\mu}{2} 2^{-n} = o(1)$ for the considered~$\mu$ yields the following.

\begin{corollary}
\label{cor:clearing-small-niches-genotypic}
The probability that the \muea with genotypic clearing, $\sigma=1$ and $\mu = o(n^{1/2})$ finds both optima on $\twomax$ in time $n^{n-2}$ is at most $o(1)$. The expected  time for finding both optima is 
$\Omega(n^{n-1})$.
\end{corollary}

As mentioned before, the use of a proper distance is really important in the context of \clearing. In our case, we use phenotypic distance for functions of unitation, which has been proved to provide more significant information at the time it is required to define small differences (in our case small niches) among individuals in a population, so the use of that knowledge can be taken into consideration at the time the algorithm is set up. Otherwise, if there is no more knowledge related to the specifics of the problem, genotypic clearing can be used but with larger niches as shown in the following section.

\section{Large Niches}
\label{sec:large_niches}

While small niches work with phenotypic clearing, Corollary~\ref{cor:clearing-small-niches-genotypic} showed that with genotypic clearing small niches are ineffective. This makes sense as for phenotypic clearing with $\sigma=1$ a niche with $i$ ones covers $\binom{n}{i}$ search points, whereas a niche in genotypic clearing with $\sigma=1$ only covers one search point. In this section we turn our attention to larger niches, where we will prove that cleared search points are likely to spread, move, and climb down a branch.

We first present general insights into these population dynamics with \clearing. These results capture the behaviour of the population in the presence of only one winning genotype~$\winner$ (of which there may be $\kappa$ copies). We estimate the time until in this situation the population evolves a search point of Hamming distance~$d$ from said winner, for any $d \le \sigma$, or for another winner to emerge (for example, in case an individual of better fitness than~$\winner$ is found). 

These time bounds are very general as they are independent of the fitness function. This is possible since, assuming the winners are fixed at~$\winner$, all other search points within the \clearingradius receive a fitness of~0 and hence are subject to a random walk. We demonstrate the usefulness of our general method by an application to \twomax with a \clearingradius of ${\sigma=n/2}$, where all winners are copies of either $0^n$ or $1^n$. The results hold both for genotypic clearing and phenotypic clearing as the phenotypic distance of any point~$x$ to $0^n$ ($1^n$, resp.) equals the Hamming distance of~$x$ to~$0^n$ ($1^n$, resp.).

\subsection{Large Population Dynamics with Clearing}
\label{sec:pop-dyn}
We assume that the population contains only one winner genotype $\winner$, of which there are $\kappa$ copies. For any given integer $0 \le \mindist \le \sigma$, we analyse the time for the population to reach a search point of Hamming distance at least $\mindist$ from $\winner$, or for a winner different from~$\winner$ to emerge. To this end, we will study a potential function $\varphi$ that measures the dynamics of the population. Let 
\[
\varphi(P_t) = \sum_{x \in P_t} \distw{x}
\] 
be the sum of all Hamming distances of individuals in the population to the winner~$\winner$. The following lemma shows how the potential develops in expectation.

\begin{lemma}
\label{lem:expected-potential}
Let $P_t$ be the current population of the \muea with genotypic clearing on any fitness function such that the only winners are $\kappa$ copies of~$\winner$ and $\distw{x} < \sigma$ for all $x \in P_t$. Then if no winner different from~$\winner$ is created, the expected change of the potential is
\[
\E{\varphi(P_{t+1}) - \varphi(P_t) \mid P_t} = 1 - \frac{\varphi(P_t)}{\mu} \left(\frac{2}{n} +
\frac{\kappa}{\mu-\kappa}\right).
\]
\end{lemma}
Before proving the lemma, let us make sense of this formula. Ignore the term $\frac{\kappa}{\mu-\kappa}$ for the moment and consider the formula $1-\frac{\varphi(P_t)}{\mu} \cdot \frac{2}{n}$. Note that $\varphi(P_t)/\mu$ is the average distance to the winner in~$P_t$. If the population has spread such that it has reached an average distance of~$n/2$ then the expected change would be $1-\frac{\varphi(P_t)}{\mu} \cdot \frac{2}{n} = 1-\frac{n}{2} \cdot \frac{2}{n} = 0$. Moreover, a smaller average distance will give a positive drift (expected value in the decrease of the distance after a single function evaluation) and an average distance larger than~$n/2$ will give a negative drift. This makes sense as a search point performing an independent random walk will attain an equilibrium state around Hamming
distance $n/2$ from~$\winner$.

The term $\frac{\kappa}{\mu-\kappa}$ reflects the fact that losers in the population do not evolve in complete isolation. The population always contains $\kappa$ copies of~$\winner$ that may create offspring and may prevent the population from venturing far away from~$\winner$. In other words, there is a constant influx of search points descending from winners~$\winner$. As the term $\frac{\kappa}{\mu-\kappa}$ indicates, this effect grows with $\kappa$, but (as we will see later) it can be mitigated by setting the population size~$\mu$ sufficiently large.

\begin{proof}[Proof of Lemma~\ref{lem:expected-potential}]
We use the notation from Algorithm~\ref{alg:mueaclea}, where $x$ is the parent, $y$ is its offspring, and $z$ is the individual considered for removal from the population.
If an individual $x \in P_t$ is selected as parent, the expected distance of its mutant to~$\winner$ is
\[
\E{\distw{y} \mid x} = \distw{x} + \frac{n-\distw{x}}{n} - \frac{\distw{x}}{n} = 1 + \distw{x}\left(1 - \frac{2}{n}\right).
\]

Hence after a uniform parent selection and mutation, the expected distance in the offspring is
\begin{equation}
\label{eq:dist-y}
\E{\distw{y} \mid P_t} = \sum_{x \in P_t} \frac{1}{\mu} \cdot \left(1 + \distw{x} \left(1 - \frac{2}{n}\right)\right)
= 1 + \frac{\varphi(P_t)}{\mu} \left(1 - \frac{2}{n}\right).
\end{equation}

After mutation and \clearing procedure, there are $\mu$ individuals in $P_t$, including $\kappa$ winners, which are copies of~$\winner$. Let $C$ denote the multiset of these $\kappa$ winners. As all $\mu-\kappa$ non-winner individuals in $P_t$ have fitness 0, one of these will be selected uniformly at random for deletion. The expected distance to~$\winner$ in the deleted individual is
\begin{equation}
\E{\distw{z} \mid P_t} = \sum_{x \in P_t \setminus C} \frac{1}{\mu-\kappa} \cdot \distw{x}
= \sum_{x \in P_t} \frac{1}{\mu-\kappa} \cdot \distw{x} = \frac{\varphi(P_t)}{\mu-\kappa}.\label{eq:dist-z}
\end{equation}

Recall that the potential is the sum of Hamming distances to~$\winner$, hence adding $y$ and removing $z$ yields $\varphi(P_{t+1}) = \varphi(P_t) + \distw{y} - \distw{z}$. Along with~\eqref{eq:dist-y} and~\eqref{eq:dist-z}, the expected change of the potential is
\begin{align*}
\E{\varphi(P_{t+1}) - \varphi(P_t) \mid P_t} =\;& \E{\distw{y} \mid P_t} - \E{\distw{z} \mid P_t}\\
=\;& 1 + \frac{\varphi(P_t)}{\mu} \left(1 - \frac{2}{n}\right) - \frac{\varphi(P_t)}{\mu-\kappa}.
\end{align*}
Using that
\[
\frac{\varphi(P_t)}{\mu} - \frac{\varphi(P_t)}{\mu-\kappa} =
\frac{(\mu-\kappa)\varphi(P_t)}{\mu(\mu-\kappa)} - \frac{\mu\varphi(P_t)}{\mu(\mu-\kappa)} = -
\frac{\kappa\varphi(P_t)}{\mu(\mu-\kappa)},
\]

the above simplifies to
\begin{align*}
\E{\varphi(P_{t+1}) - \varphi(P_t) \mid P_t} =\;& 1 - \frac{2\varphi(P_t)}{\mu n} - \frac{\kappa\varphi(P_t)}{\mu(\mu-\kappa)}\\
=\;& 1 - \frac{\varphi(P_t)}{\mu} \left(\frac{2}{n} + \frac{\kappa}{\mu-\kappa}\right).\qedhere
\end{align*}
\end{proof}

The potential allows us to conclude when the population has reached a search point of distance at least~$\mindist$ from~$\winner$. The following lemma gives a sufficient condition.

\begin{lemma}
\label{lem:sufficient-condition-for-reaching-distance}
If $P_t$ contains $\kappa$ copies of $\winner$ and $\varphi(P_t) > (\mu-\kappa)(\mindist-1)$ then $P_t$ must contain at least one individual $x$ with $\distw{x} \ge \mindist$.
\end{lemma}
\begin{proof}
There are at most $\mu-\kappa$ individuals different from~$\winner$. By the pigeon-hole principle, at least one of them must have at least distance $\mindist$ from~$\winner$.
\end{proof}

In order to bound the time for reaching a high potential given in
Lemma~\ref{lem:sufficient-condition-for-reaching-distance}, we will use the following drift theorem, a straightforward extension of the variable drift theorem~\citep{Johannsen2010} towards reaching any state smaller than some threshold~$a$. It can be derived with simple adaptations to the proof in~\cite{Rowe2014}.

\begin{theorem}[Generalised variable drift theorem]
\label{drift:johannsen-general}
Consider a stochastic process $X_0, X_1, \dots$ on $\{0, 1, \dots, m\}$. Suppose there is a monotonic increasing function $h: \mathbb{R}^+ \rightarrow \mathbb{R}^+$ such that
\[
\E{X_t - X_{t+1} \mid X_t = k} \geq h(k)
\]
for all $k \in \{a, \dots, m\}$. Then the expected first hitting time of the set $\{0,\dots,a-1\}$ for $a \in \mathbb{N}$ is at most
\[
\frac{a}{h(a)} + \int_{a}^m \frac{1}{h(x)} \;\mathrm{d}x.
\]
\end{theorem}

The following lemma now gives an upper bound on the first hitting time (the random variable that denotes the first point in time to reach a certain point) of a search point with distance at least~$\mindist$ to the winner~$\winner$.
\begin{lemma}
\label{lem:time-to-clearing-radius}
Let $P_t$ be the current population of the \muea with genotypic clearing and $\sigma \le n/2$ on any fitness function such that $P_t$ contains $\kappa$ copies of a unique winner~$\winner$ and $\distw{x} < \mindist$ for all $x \in P_t$. For any $0 \le \mindist \le \sigma$, if $\mu \ge \kappa \cdot \frac{\mindist n - 2\mindist + 2}{n-2\mindist+2}$ then the expected time until a search point $x$ with $\distw{x} \ge \mindist$ is found, or a winner different from~$\winner$ is created, is $O(\mu n \log \mu)$.
\end{lemma}
\begin{proof}
We pessimistically assume that no other winner is created and estimate the first hitting time of a search point with distance at least~$\mindist$. As $\varphi$ can only increase by at most~$n$ in one step, $h_{\max} := (\mu-\kappa)(\mindist-1)+n$ is an upper bound on the maximum potential that can be achieved in the generation where a distance of~$\mindist$ is reached or exceeded for the first time.

In order to apply drift analysis, we define a distance function that describes how close the algorithm is to reaching a population where a distance~$\mindist$ was reached. We consider the random walk induced by $X_t := h_{\max} - \varphi(P_t)$, stopped as soon as a Hamming distance of at least~$\mindist$ from~$\winner$ is reached. Due to our definition of $h_{\max}$, the random walk only attains values in $\{0, \dots, h_{\max}\}$ as required by the variable drift theorem.

By Lemma~\ref{lem:expected-potential}, abbreviating $\alpha := \frac{1}{\mu} \left(\frac{2}{n} + \frac{\kappa}{\mu-\kappa}\right)$, $X_t$ decreases in expectation by at least $h(P_t) := 1-\alpha \varphi(P_t) = 1 - \alpha h_{\max} + \alpha h(P_t)$, provided $h(P_t) > 0$. By definition of $h$ and Lemma~\ref{lem:sufficient-condition-for-reaching-distance}, the population reaches a distance of at least~$\mindist$ once the distance~$h_{\max} - \varphi(P_t)$ has dropped below $n$. Using the generalised variable drift theorem, the expected time till this happens is at most
\[
\frac{n}{1-\alpha h_{\max} + \alpha n} + \int_n^{h_{\max}} \frac{1}{1-\alpha h_{\max} + \alpha x} \mathrm{d}x.
\]

Using $\int \frac{1}{ax+b} \;\mathrm{d}x = \frac{1}{a} \ln\left|ax+b\right|$ \cite[Equation 3.3.15]{Abramowitz1974}, we get
\begin{align*}
\;& \frac{n}{1-\alpha h_{\max} + \alpha n} + \left[ \frac{1}{\alpha} \ln(1-\alpha h_{\max} + \alpha
x)\right]_n^{h_{\max}}\\
=\;& \frac{n}{1-\alpha h_{\max} + \alpha n} + \frac{1}{\alpha} \cdot \left(\ln(1) - \ln(1-\alpha h_{\max} +
\alpha n)\right)\\
=\;& \frac{n}{1-\alpha h_{\max} + \alpha n} + \frac{1}{\alpha} \ln((1-\alpha h_{\max} + \alpha n)^{-1}).
\end{align*}
We now bound the term $1-\alpha h_{\max} + \alpha n$ from below as follows.
\begin{align*}
1-\alpha h_{\max} + \alpha n=\;& 1 - (\mu-\kappa)(\mindist-1) \cdot \frac{1}{\mu}\left(\frac{2}{n} + \frac{\kappa}{\mu-\kappa}\right)\\
=\;& 1 - \frac{2(\mu-\kappa)(\mindist-1)+ \kappa(\mindist-1)n}{\mu n}\\
=\;& \frac{\mu n - 2\mu\mindist +2\kappa \mindist - \kappa \mindist n + 2\mu - 2\kappa + \kappa n}{\mu n}\\
=\;& \frac{\kappa}{\mu} + \frac{n - 2\mindist + 2}{n} -
\frac{\kappa \mindist n - 2\kappa\mindist + 2\kappa}{\mu n}\\
\ge\;& \frac{\kappa}{\mu} + \frac{n - 2\mindist + 2}{n} - \frac{n - 2\mindist + 2}{n}\\
=\;& \frac{\kappa}{\mu}
\end{align*}
where in the penultimate step we used the assumption $\mu \ge \kappa \cdot \frac{\mindist n - 2\mindist + 2}{n-2\mindist+2}$. Along with $\alpha \ge 2/(\mu n)$, the expected time bound simplifies to
\[
\frac{n}{1-\alpha h_{\max} + \alpha n} + \frac{1}{\alpha} \ln((1-\alpha h_{\max} + \alpha n)^{-1}) \le \frac{n}{\kappa/\mu} + \frac{\mu n}{2} \ln(\mu/\kappa) = O(\mu n \log \mu).
\qedhere
\]
\end{proof}

The minimum threshold $\kappa \cdot \frac{\mindist n - 2\mindist + 2}{n-2\mindist+2}$ for $\mu$ contains a factor of~$\kappa$. The reason is that the fraction of winners in the population needs to be small enough to allow the population to escape from the vicinity of~$\winner$. The population size hence needs to grow proportionally to the number of winners~$\kappa$ the population is allowed to store.

Note that the restriction $\mindist \le \sigma \le n/2$ is necessary in Lemma~\ref{lem:time-to-clearing-radius}. Individuals evolving within the \clearingradius, but at a distance larger than $n/2$ to~$\winner$ will be driven back towards~$\winner$. If $\mindist$ is significantly larger than~$n/2$, we conjecture that the expected time for reaching a distance of at least~$\mindist$ from~$\winner$ becomes exponential in~$n$.

\subsection{Upper Bound for \twomax}
\label{sec:upper-bound-twomax}
It is now easy to apply Lemma~\ref{lem:time-to-clearing-radius} in order to achieve a running time bound on \twomax. Putting $\mindist=\sigma=n/2$, the condition on~$\mu$ simplifies to
\[
\mu \ge \kappa \cdot \frac{\mindist n - 2\mindist + 2}{n-2\mindist+2}= \kappa \cdot \frac{n^2/2 - n + 2}{2}, 
\]
which is implied by $\mu \ge \kappa n^2/4$. Lemma~\ref{lem:time-to-clearing-radius} then implies the following. Recall that for $\winner \in \{0^n, 1^n\}$, genotypic distances $\distw{x}$ equal phenotypic distances, hence the result applies to both genotypic and phenotypic clearing.
\begin{corollary}
\label{cor:fitness-valley-crossing-on-twomax}
Consider the \muea with genotypic or phenotypic clearing, $\kappa \in \mathbb{N}, \mu \ge \kappa n^2/4$ and $\sigma=n/2$ on \twomax with a population containing $\kappa$ copies of $0^n$ ($1^n$). Then the expected time until a search point with at least (at most) $n/2$ ones is found is $O(\mu n \log \mu)$.
\end{corollary}

\begin{theorem}
\label{the:few-niches-twomax}
The expected time for the \muea with genotypic or phenotypic clearing, $\mu \ge \kappa n^2/4$, 
$\mu \le \poly{n}$ and $\sigma=n/2$ finding both optima on \twomax is $O(\mu n \log n)$.
\end{theorem}
\begin{proof}
We first estimate the time to reach one optimum, $0^n$ or $1^n$. The population is elitist as it always contains a winner with the best-so-far fitness. Hence we can apply the level-based argument as follows. If the current best fitness is $i$, it can be increased by selecting an individual with fitness $i$ (probability at least $1/\mu$) and flipping only one of $n-i$ bits with the minority value (probability at least $(n-i)/(en)$). The expected time for increasing the best fitness~$i$ is hence at most $\mu \cdot en/(n-i)$ and the expected time for finding some optimum $\winner \in \{0^n, 1^n\}$ is at most
\[
\sum_{i=n/2}^{n-1} \mu \cdot \frac{en}{n-i} = e\mu n \sum_{i=1}^{n/2} \frac{1}{i} \le e\mu n \ln n.
\]

In order to apply Corollary~\ref{cor:fitness-valley-crossing-on-twomax}, we need to have $\kappa$ copies of~$\winner$ in the population. While this isn't the case, a generation picking $\winner$ as parent and not flipping any bits creates another winner $\winner$ that will remain in the population. If there are $j$ copies of~$\winner$, the probability to create another winner is at least $j/\mu \cdot (1-1/n)^{n} \ge j/(4\mu)$ (using $n \ge 2$). Hence the time until the population contains $\kappa$ copies of~$\winner$ is at most
\[
\sum_{j=1}^\kappa \frac{4\mu}{j} = O(\mu \log \kappa) = O(\mu \log n)
\]
as $\kappa \le \mu \le \poly{n}$.

By Corollary~\ref{cor:fitness-valley-crossing-on-twomax}, the expected time till a search point on the opposite branch is created is $O(\mu n \log \mu) = O(\mu n \log n)$. Since the best individual on the opposite branch is a winner in its own niche, it will never be removed. This allows the population to climb this branch as well. Repeating the arguments from the first paragraph of this proof, the expected time till the second optimum is found is at most $e\mu n \ln n$. Adding up all expected times proves the claim. 
\end{proof}

One limitation of Theorem~\ref{the:few-niches-twomax} is the steep requirement on the population size: $\mu \ge \kappa n^2/4$. The condition on $\mu$ was chosen to ensure a positive drift of the potential for all populations that haven't reached distance~$\mindist$ yet, including the most pessimistic scenario of all losers having distance $\mindist-1$ to~$\winner$. Such a scenario is unlikely as we will see in Sections~\ref{sec:empangen} and \ref{sec:empanpopsize} where experiments suggest that the population tends to spread out, covering a broad range of distances. With such spread, a distance of $\mindist$ can be reached with a much smaller potential than that indicated by Lemma~\ref{lem:sufficient-condition-for-reaching-distance}. We conjecture that the \muea with clearing is still efficient on \twomax if $\mu=O(n)$. However, proving this theoretically may require new arguments on the distribution of the $\kappa$ winners and losers inside the population.

\subsection{On the Choice of the Population Size for \twomax}
\label{sec:smallpopdyn}

To get further insights into what population sizes~$\mu$ are necessary, we show in the following that the \muea with \clearing becomes inefficient on \twomax if $\mu$ is too small, that is, smaller than $n/\polylog{n}$. The reason is as follows: assume that the population only contains a single optimum $x^*$, and further individuals that are well within a niche of size $\sigma=n/2$ surrounding $x^*$. Due to \clearing, the population will always contain a copy of $x^*$. Hence there is a constant influx of individuals that are offspring, or, more generally, recent descendants of $x^*$. We refer to these individuals informally as \emph{young}; a rigorous notation will be provided in the proof of Theorem~\ref{the:smallpoptakeover}. Intuitively, young individuals are similar to~$x^*$, and thus are likely to produce further offspring that are also \emph{young}, \ie similar to~$x^*$ when chosen as parents.

We will show in the following that if the population size~$\mu$ is small, young individuals will frequently take over the whole population, creating a population where all individuals are similar to~$x^*$. This takeover happens much faster than the time the algorithm needs to evolve a lineage that can reach a Hamming distance $n/2$ to the optimum. 

The following theorem shows that if the population size is too small, the \muea is unable to escape from one local optimum, assuming that it starts with a population of search points that have recently evolved from said optimum. 

\begin{theorem}
\label{the:smallpoptakeover}
Consider the \muea with genotypic or phenotypic clearing on \twomax with $\mu\leq n/(4\log^3 n)$, $\kappa=1$ and $\sigma=n/2$, starting with a population containing only search points that have evolved from one optimum $x^*$ within the last $\mu n/32$ generations. Then the probability that both optima are found within time $n^{(\log n)/2}$ is $n^{-\Omega(\log n)}$.
\end{theorem}

The following lemma describes a stochastic process that we will use in the proof of Theorem~\ref{the:smallpoptakeover} to model the number of ``young'' individuals over time. We are interested in the first hitting time of state~$\mu$ as this is the first point in time where young individuals have taken over the whole population of size~$\mu$. The transition probabilities for states $1 < X_t < \mu$ reflect the evolution of a fixed-size population containing two species (young and old in our case): in each step one individual is selected for reproduction, and another individual is selected for replacement. If they stem from the same species, the size of both species remains the same. But if they stem from different species, the size of the first species can increase or decrease by 1, with equal probability. 

This is similar to the \emph{Moran process} in population genetics~\cite[Section~3.4]{ewens2004mathematical} which ends when one species has evolved to fixation (i.\,e.\ has taken over the whole population) or extinction. Our process differs as state~1 is reflecting, hence extinction of young individuals is impossible. Notably, we will show that, compared to the original Moran process, the expected time for the process to end is larger by a factor of order $\log \mu$. Other variants of the Moran process have also appeared in different related contexts such as the analysis of Genetic Algorithms (Lemma~6 in~\citealp{Dang2016a}) and the analysis of the compact Genetic Algorithm (Lemma~7 in~\citealp{Sudholt2016a}). The following lemma gives asymptotically tight bounds on the time young individuals need to evolve to fixation.
\begin{lemma}
\label{lem:smallpoptakeover}
Consider a Markov chain $\{X\}_{t \ge 0}$ on $\{1, 2, \dots, \mu\}$ with transition probabilities\linebreak
$\Prob{X_{t+1}=X_t+1\mid X_t} = X_t(\mu-X_t)/\mu^2$ for $1 \le X_t < \mu$,
$\Prob{X_{t+1}=X_t-1\mid X_t} = X_t(\mu-X_t)/\mu^2$ for $1 < X_t < \mu$ and
$X_{t+1} = X_t$ with the remaining probability.
Let $T$ be the first hitting time of state 
$\mu$, then for all starting states $X_0$,
\[
\frac{1}{2} \cdot (\mu-X_0)\mu \ln(\mu-1) \le \E{T\mid X_0} \le 4(\mu-X_0)\mu \harm(\mu/2) \le 4\mu^2 \ln{\mu}. 
\]
In addition, if $\mu \le n$ then $\Prob{T \ge 8\mu^2 \log^3 n} \le n^{-\log n}$.
\end{lemma}
\begin{proof}
Let us abbreviate $\e_i=\E{T\mid X_t=i}$, then $\e_\mu=0$, $\e_1=\frac{\mu^2}{\mu-1}+\e_2$, and for all $1<i<\mu$ we have 
\begin{align*}
\e_i&=1+\frac{i(\mu-i)}{\mu^2}\cdot\e_{i+1}+\frac{i(\mu-i)}{\mu^2}\cdot\e_{i-1}+\left(1-\frac{2i(\mu-i)}{\mu^2}\right)\cdot\e_i\\
\Leftrightarrow \ \frac{2i(\mu-i)}{\mu^2}\cdot\e_i&=1+\frac{i(\mu-i)}{\mu^2}\cdot\e_{i+1}+\frac{i(\mu-i)}{\mu^2}\cdot\e_{i-1}\\
\Leftrightarrow \ 2\e_i&=\frac{\mu^2}{i(\mu-i)}+ \e_{i+1}+ \e_{i-1}\\
\Leftrightarrow \ \e_i-\e_{i+1}&=\frac{\mu^2}{i(\mu-i)}+\e_{i-1}-\e_i.
\end{align*}
Introducing $D_i:=\e_i-\e_{i+1}$, this is
\begin{align*}
D_i&=\frac{\mu^2}{i(\mu-i)}+D_{i-1}.
\end{align*}
For $D_1$ we get 
\[
D_1=\e_1-\e_2=\left(\frac{\mu^2}{\mu-1}+\e_2\right)-\e_2=\frac{\mu^2}{\mu-1}.
\]
More generally, we expand $D_i$ to get
\[
D_i=\sum_{j=1}^{i}\frac{\mu^2}{j(\mu-j)} =\mu^2\sum_{j=1}^{i}\frac{1}{j(\mu-j)}
\]
Now we can express $\e_i$ in terms of $D_i$ variables as follows. For all $1 \le i <\mu$, 
\[
\e_i=\underbrace{(\e_i-\e_{i+1})}_{D_i}+\underbrace{(\e_{i+1}-\e_{i+2})}_{D_{i+1}}+\cdots+\underbrace{(\e_{\mu-1}-\e_{\mu})}_{D_{\mu-1}}+\underbrace{\e_{\mu}}_{0}
\]
hence
\[
\e_i=D_i+\ldots+D_{\mu-1}= \mu^2\sum_{k=i}^{\mu-1} \sum_{j=1}^{k}\frac{1}{j(\mu-j)}
\]
We now bound this double-sum from above and below.
\begin{align*}
\mu^2 \sum_{k=i}^{\mu-1} \sum_{j=1}^{k}\frac{1}{j(\mu-j)}
\le\;& \mu^2 \sum_{k=i}^{\mu-1} \left(\sum_{j=1}^{\lfloor \mu/2 \rfloor}\frac{1}{j(\mu-j)} + \sum_{j=\lfloor \mu/2 \rfloor+1}^{k}\frac{1}{j(\mu-j)}\right)\\
\le\;& \mu^2 \sum_{k=i}^{\mu-1} \left(\sum_{j=1}^{\lfloor \mu/2 \rfloor}\frac{1}{j(\mu-j)} + \sum_{j=1}^{\lfloor \mu/2 \rfloor}\frac{1}{j(\mu-j)}\right)\\
\leq\;& \mu^2 \sum_{k=i}^{\mu-1} \frac{4}{\mu} \cdot \harm(\lfloor \mu/2 \rfloor)
= 4(\mu-i)\mu \harm(\lfloor \mu/2 \rfloor).
\end{align*}
The final inequality follows from $(\mu-i)\cdot\harm(\lfloor \mu/2 \rfloor)\leq\mu \ln \mu$.

The lower bound follows from 
\begin{align*}
\mu^2 \sum_{k=i}^{\mu-1} \sum_{j=1}^{k}\frac{1}{j(\mu-j)}
\geq \mu \sum_{k=i}^{\mu-1} \sum_{j=1}^{k}\frac{1}{j}
\geq \mu \sum_{k=i}^{\mu-1} \ln(k)
=\;& \mu \ln\left(\prod_{k=i}^{\mu-1} k\right)\\
\geq\;& \mu \ln\left((\mu-1)^{(\mu-i)/2}\right)\\
=\;& \frac{1}{2} \cdot (\mu-i)\mu \ln(\mu-1).
\end{align*}
where in the last inequality we used that $(\mu-1-j)(i+j) \ge \mu-1$ for all $0 \le j \le \mu-1$, allowing us to group factors in $\lfloor (\mu-i)/2 \rfloor$ pairs whose product is at least $\mu-1$, leaving a remaining factor of at least $\mu/2 \ge (\mu-1)^{1/2}$ if $(\mu-i)$ is odd.

For the second statement, we use standard arguments on independent phases. By Markov's inequality, the probability that takeover takes longer than $2\cdot(4\mu^2\ln{\mu})$ is at most $1/2$. Since the upper bound holds for any $X_0$, we can iterate this argument $\log^2 n$ times. Then the probability that we do not have a takeover in $2\cdot(4\mu^2\ln{\mu}\cdot\log^2 n) \le 8\mu^2 \log^3 n$ steps (using $\mu \le n$) is $2^{-\log^2 n}=n^{-\log n}$.
\end{proof}

Now we prove that the time required to reach a new niche with $\sigma=n/2$ is larger than the time required for ``young'' individuals to take over the population. In other words, once a winner $x^{*}$ is found and assigned to an optimum, with a small $\mu$, the time for a takeover is shorter than the required time to find a new niche. This will imply that the algorithm needs superpolynomial time to escape from the influence of the winner $x^{*}$ and in consequence it needs superpolynomial time to find the opposite optimum. 

We analyse the dynamics within the population by means of so-called \emph{family trees}. The analysis of EAs with family trees has been introduced by \cite{Witt2006} for the analysis of the \muea. According to Witt, a family tree is a directed acyclic graph whose nodes represent individuals and edges represent direct parent-child relations created by a mutation-based EAs. After initialisation, for every initial individual $r^{*}$ there is a family tree containing only $r^{*}$. We say that $r^{*}$ is the root of the family tree $T(r^{*})$. Afterwards, whenever the algorithm chooses an individual $x\in T(r^{*})$ as parent and creates an offspring $y$ out of $x$, a new node representing $y$ is added to $T(r^{*})$ along with an edge from $x$ to $y$. That way, $T(r^{*})$ contains all descendants from $r^{*}$ obtained by direct and indirect mutations.

There may be family trees containing only individuals that have been deleted from the current population. As $\mu$ individuals survive in every selection, at least one tree is guaranteed to grow. A subtree of a family tree is, again, a family tree. A (directed) path within a family tree from $x$ to $y$ represents a sequence of mutations creating $y$ out of $x$. The number of edges on a longest path from the root $x^*$ to a leaf determines the depth of $T(r^*)$.

\cite{Witt2006} showed how to use family trees to derive lower bounds on the optimisation time of mutation-based EAs. Suppose that after some time $t$ the depth of a family tree $T(r^*)$ is still small. Then typically the leaves are still quite similar to the root. Here we make use of Lemma~1 in \cite{Sudholt2009} (which is an adaptation from Lemma~2 and proof of Theorem~4 in \citealp{Witt2006}) to show that the individuals in $T(r^*)$ are still concentrated around $r^*$. If the distance from $r^*$ to all optima is not too small, then it is unlikely that an optimum has been found after $t$ steps.

\begin{lemma}[Adapted from Lemma~1 in \citealp{Sudholt2009}]
\label{lem:famtree}
For the \muea with or without clearing, let $r^{*}$ be an individual entering the population in some generation $t^{*}$. The probability that within the following $t$ generations some $y^{*}\in T(r^{*})$ emerges with $\Hamm{r^*}{y^*}\geq 8t/\mu$ is $2^{-\Omega(t/\mu)}$.
\end{lemma}

Lemma~1 in \cite{Sudholt2009} applies to ($\mu$+$\lambda$)~EAs without \clearing. We recap Witt's basic proof idea to make the paper self-contained and also to convince the reader why the result also applies to the \muea with \clearing.

The analysis is divided in two parts. In the first part it is shown that family trees are unlikely to be very deep. Since every individual is chosen as parent with probability $1/\mu$, the expected length of a path in the family tree after $t$ generations is bounded by $t/\mu$. Large deviations from this expectation are unlikely. Lemma~2 in~\cite{Witt2006} shows that the probability that a family tree has depth at least $3t/\mu$ is $2^{-\Omega(t/\mu)}$. This argument only relies on the fact that parents are chosen uniformly at random, which also holds for the \muea with \clearing.

For family trees whose depth is bounded by $3t/\mu$, all path lengths are bounded by $3t/\mu$. Each path corresponds to a sequence of standard bit mutations, and the Hamming distance between any two search points on the same path can be bounded by the number of bits flipped in all mutations that lead from one search point to the other.
By applying Chernoff bounds (see \citealp{Motwani1995}) with respect to the upper bound $4t/\mu$ on the expectation instead of the expectation itself (cf.~\citealp[page 75]{Witt2006}), we obtain that the probability of an individual of Hamming distance at least $8t/\mu$ to $r^*$ emerging on a particular path is at most $e^{-4t/(3\mu)}$. Taking the union bound over all possible paths in the family tree still gives a failure probability of $2^{-\Omega(t/\mu)}$. Adding the failure probabilities from both parts proves the claim.

Now, Lemma~\ref{lem:famtree} implies the following corollary.

\begin{corollary}
\label{cor:lineagehamwitt}
The probability that, starting from a search point~$x^*$, within $\mu n/16$ generations the \muea with clearing evolves a lineage that reaches Hamming distance at least $n/2$ to its founder~$x^*$ is $2^{-\Omega(n)}$.
\end{corollary}

Now we put Lemma~\ref{lem:smallpoptakeover} and Corollary~\ref{cor:lineagehamwitt} together to prove Theorem~\ref{the:smallpoptakeover}.
\begin{proof}[Proof of Theorem~\ref{the:smallpoptakeover}]
By assumption, all individuals in the population are descendants of individuals with genotype $x^*$, and this property will be maintained over time. This means that every individual $x$ in the population $P_t$ at time~$t$ will have an ancestor that has genotype~$x^*$ (our notion of \emph{ancestor} and \emph{descendant} includes the individual itself). Tracing back $x$'s ancestry, let $t^* \le t$ be the most recent generation where an ancestor of $x$ has genotype~$x^*$. Then we define the \emph{age} of $x$ as $t - t^*$. Informally, the age describes how much time a search point has had to evolve differences from the genotype~$x^*$.
Note that the age of $x^*$ itself is always 0 and as the population always contains a winner~$x^*$, it always contains at least one individual of age~0.

Now assume that a new search point $x$ is created with $\Hamm{x}{x^*} \ge n/2$. If $x$ has age at most $\mu n/16$ then there exists a lineage from a copy of $x^*$ to~$x$ that has emerged in at most $\mu n/16$ generations. This corresponds to the event described in Corollary~\ref{cor:lineagehamwitt}, and by said corollary the probability of this event happening is at most $2^{-\Omega(n)}$. Taking the union bound over all family trees (of which there are at most~$\mu$ in every generation) and the first $n^{\log n}$ generations, the probability that such a lineage does emerge in any family tree and at any point in time within the considered time span is still bounded by $\mu n^{\log n} \cdot 2^{-\Omega(n)} = 2^{-\Omega(n)}$.

We now show using Lemma~\ref{lem:smallpoptakeover} that it is very unlikely that individuals with age larger than $\mu n/16$ emerge. We say that a search point $x$ is \emph{$T$-young} if it has genotype~$x^*$ or if its most recent ancestor with genotype~$x^*$ was born during or after generation~$T$. Otherwise, $x$ is called \emph{$T$-old}. We omit the parameter ``$T$'' whenever the context is obvious. A key observation is that youth is inheritable: if a young search point is chosen as parent, then the offspring is young as well. If an old search point is chosen as parent, then the offspring is old as well, unless mutation turns the offspring into a copy of~$x^*$.

Let $X_t$ be the number of young individuals in the population at time~$t$, and pessimistically ignore the fact that old individuals may create young individuals through lucky mutations. Then in order to increase the number of young individuals, it is necessary and sufficient to choose a young individual as parent (probability $X_t/\mu$) and to select an old parent for replacement. The probability of the latter is $(\mu-X_t)/\mu$ as there are $\mu-X_t$ old parents and the individual to be removed is chosen uniformly at random among $\mu$ individuals whose fitness is cleared. Hence, for $1 \le X_t < \mu$, $\Prob{X_{t+1}=X_t+1\mid X_t} = X_t(\mu-X_t)/\mu^2$. Similarly, the number of old individuals increases if and only if an old individual is chosen as parent (probability $(\mu-X_t)/\mu$) and a young individual is chosen for replacement (probability $X_t/\mu$), hence for $1 < X_t < \mu$ we have $\Prob{X_{t+1}=X_t-1\mid X_t} = X_t(\mu-X_t)/\mu^2$. Otherwise, $X_{t+1} = X_t$. Note that $X_t \ge 1$ since the winner~$x^*$ is young and will never be removed. This matches the Markov chain analysed in Lemma~\ref{lem:smallpoptakeover}. 

Now consider a generation $T$ where all individuals in the population have ages at most $\mu n/32$. By assumption, this property is true for the initial population. At time~$T$, the population contains at least one $T$-young individual: the winner~$x^*$. By Lemma~\ref{lem:smallpoptakeover}, with probability at least $1-n^{-\log n}$, within the next $8\mu^2 \log^3 n \le \mu n/32$ generations, using the condition $\mu\leq n/(4\log^3 n)$, the population will reach a state where $X_t = \mu$, that is, all individuals are $T$-young. Assuming this does happen, let $T' \le T + \mu n/32$ denote the first point in time where this happens. Then at time $T'$ all individuals have ages at most $\mu n/32$, and we can iterate the above arguments with $T'$ instead of~$T$.

Each such iteration carries a failure probability of at most $n^{-\log n}$. Taking the union bound over failure probabilities $n^{-\log n}$ over the first $n^{(\log n)/2}$ generations yields that the probability of an individual of age larger than $\mu n/16$ emerging is only $n^{(\log n)/2} \cdot n^{-\log n} = n^{-(\log n)/2}$.

Adding failure probabilities $2^{-\Omega(n)}$ and $n^{-(\log n)/2}$ completes the proof.
\end{proof}
 
We conjecture that a population size of $\mu=O(n)$ is sufficient to optimise $\twomax$ in expected time $O(\mu n \log n)$, that is, that the conditions in Theorem~\ref{the:few-niches-twomax} can be improved.

\section{Generalisation to Other Example Landscapes}
\label{sec:genexaland}

Note that, in contrast to previous analyses of \emph{fitness sharing}~\citep{Friedrich2009,Oliveto2014a}, our analysis of the \clearing mechanism does not make use of the specific fitness values of \twomax. The main argument of how to escape from one local optimum only depends on the size of its basin of attraction. Our results therefore easily extend to more general function classes that can be optimised by leaving a basin of attraction of width at most $n/2$.

We consider more general classes of examples landscapes introduced by \cite{Jansen2016} addressing the need for suitable benchmark functions for the theoretical analysis of evolutionary algorithms on multimodal functions. Such benchmark functions allows the control of different features like the number of peaks (defined by their position), their slope and their height (provided in an indirect way), while still enabling a theoretical analysis. Since this benchmark setting is defined in the search space $\{0,1\}^n$ and it uses the Hamming distance between two bit strings, it matches perfectly with the current investigation.

\cite{Jansen2016} define their notion of a landscape as the number of peaks $k\in \mathbb{N}$ and the definition of the $k$ peaks (numbered $1,2,\ldots,k$) where the $i$-th peak is defined by its position $p_i\in \{0,1\}^n$, its slope $a_i\in \mathbb{R}^{+}$, and its offset $b_i\in \mathbb{R}_{0}^{+}$. The general idea is that the fitness value of a search point depends on peaks in its vicinity. The main objective for any optimisation algorithm operating in this landscape is to identify those peaks: a highest peak in exact optimisation or a collection of peaks in multimodal optimisation. A peak has been identified or reached if the Hamming distance of a search point $x$ and a peak $p_i$ is $\Hamm{x}{p_i}=0$. Since we are considering maximisation, it is more convenient to consider $\G{x}{p_i}:=n-\Hamm{x}{p_i}$ instead.

There are three different fitness functions used to deal with multiple peaks in \cite{Jansen2016}; we consider the two most interesting function classes $f_1$ and $f_2$ defined in the following.
We only consider genotypic clearing in the following as phenotypic clearing only makes sense for functions of unitation.

\begin{definition}[Definition~3 in \citealp{Jansen2016}]
\label{def:fam_func}
Let $k\in \mathbb{N}$ and $k$ peaks $(p_1,a_1,b_1),(p_2,a_2,\allowbreak b_2),\ldots,(p_k,a_k,b_k)$ be given, then
\begin{itemize}
\item $f_1(x):=a_{\CP{x}}\cdot \G{x}{p_{\CP{x}}}+b_{\CP{x}}$, called the nearest peak function.
\item $f_2(x):=\displaystyle\max_{i\in \{1,2,\ldots,k\}} a_{i}\cdot \G{x}{p_{i}}+b_{i}$, called the weighted nearest peak function.
\end{itemize}
where $\CP{x}$ is defined by the closest peak (given by its index $i$) to a search point 
\[
\CP{x}:=\argmin_{i\in \{1,2,\ldots,k\}} \Hamm{x}{p_i}.
\]
\end{definition}

The nearest peak function, $f_1$, has the fitness of a search point $x$ determined by the closest peak $i = \CP{x}$ that determines the slope $a_i$ and the offset $b_i$. In cases where there are multiple $i$ that minimise $\Hamm{x}{p_i}$, $i$ should additionally maximise $a_i \cdot \G{x}{p_i} + b_i$. If there is still not a unique individual, a peak $i$ is selected uniformly at random from those that minimises $\Hamm{x}{p_i}$ and those that maximises $a_i \cdot \G{x}{p_i} + b_i$.

The weighted nearest peak function, $f_2$, takes the height of peaks into account. It uses the  peak $i$ that yields the largest value to determine the function value. The bigger the height of the peak, the bigger its influence on the search space in comparison to smaller peaks.

\subsection{Nearest Peak Functions}
\label{sec:nepefu}

We first argue that our results easily generalise to nearest peak functions with two complementary peaks $p_2 = \overline{p_1}$, arbitrary slopes $a_1, a_2 \in \mathbb{R}^{+}$, and arbitrary offsets $b_i\in \mathbb{R}_{0}^{+}$. The generalisation from peaks $0^n, 1^n$ as for \twomax to peaks $p_2 = \overline{p_1}$ is straightforward: we can swap the meaning of zeros and ones for any selection of bits without changing the behaviour of the algorithm, hence the \muea with \clearing will show the same stochastic behaviour on peaks $0^n, 1^n$ as well as on arbitrary peaks $p_2 = \overline{p_1}$. As for \twomax, if only one peak $x^*$ has been found, the basin of attraction of the other peak is found once a search point with Hamming distance at least~$n/2$ to~$x^*$ is generated. If the \clearingradius is set to $\sigma = n/2$, the \muea with \clearing will create a new niche, and from there it is easy to reach the complementary optimum $\overline{x^*}$. In fact, our analyses from Section~\ref{sec:large_niches} never exploited the exact fitness values of \twomax; we only used information about basins of attraction, and that it is easy to locate peaks via hill climbing. We conclude our findings in the following corollary.

\begin{corollary}
\label{cor:generalisation_twomax}
The expected time for the \muea with genotypic clearing, $\kappa\in\mathbb{N}$, $\mu \geq \kappa n^2/4$, $\mu\leq\poly{n}$ and $\sigma = n/2$ finding both peaks on any nearest peak function $f_1$ with two complementary peaks $p_2 = \overline{p_1}$ is $O(\mu n \log{n})$.

If $\mu\leq n/(4\log^3 n)$, $\kappa=1$ and $\sigma=n/2$, and the \muea starts with a population containing only search points that have evolved from one optimum $x^*$ within the last $\mu n/32$ generations, the probability that both optima are found within time $n^{(\log n)/2}$ is $n^{-\Omega(\log n)}$.
\end{corollary}

\subsection{Weighted Nearest Peak Functions}
\label{sec:weneapea}
For $f_2$ things are different: the larger the peak, the larger is its influence area of the search space in comparison to smaller peaks and thus will have a larger basin of attraction. These asymmetric variants with suboptimal peaks with smaller basin of attraction and peaks with larger basin of attraction are similar to the analysis made in Section~\ref{sec:nepefu}, as long as the parameter $\sigma$ is set as the maximum distance between the peaks necessary to form as many niches as there are peaks in the solution, and the restriction $0\leq d\leq n/2$ of Lemma~\ref{lem:time-to-clearing-radius} is met, the same analysis can be applied for this instance of the family of landscapes benchmark.

According to $f_2$ in Definition~\ref{def:fam_func}, the bigger the height of the peak, the bigger its influence on the search space in comparison to the smaller peaks. Let $\B_i$ denote the basin of attraction of the highest peak $p_i$, as long as $0\leq\B_i\leq n/2$ from Lemma~\ref{lem:time-to-clearing-radius} it will be possible to escape from the influence of $p_i$ and create a new winner from a new niche with distance $\Hamm{x}{p_i}\geq \B_i$. Jansen and Zarges show~\cite[Theorem~2]{Jansen2016} that for two complementary peaks $p_2 = \overline{p_1}$ the basin of attraction of $p_1$ contains all search points $x$ with 
\[
n - \Hamm{x}{p_1} < \frac{a_2}{a_1 + a_2} \cdot n + \frac{b_2 - b_1}{a_1 + a_2}.
\]
Using that the peaks are complementary, a symmetric statement holds for $\B_2$. Note that in the special case of $a_1 = a_2$ and $b_1 = b_2$ the right-hand side simplifies to~$n/2$.

Along with our previous upper bound on \twomax from Theorem~\ref{the:few-niches-twomax} it is easy to show the following result for a large class of weighted nearest peak functions~$f_2$.
\begin{theorem}
\label{the:fam_f2}
For all weighted nearest peak functions $f_2$ with two complementary peaks $p_2 = \overline{p_1}$ meeting the following conditions on $a_1,a_2 \in \mathbb{R}^+$ and $b_1,b_2\in \mathbb{R}_0^+$ and the clearing radius~$\sigma$
\begin{align*}
f_2(p_1) \le f_2(p_2) &\;\Rightarrow\; \frac{a_1}{a_1 + a_2} \cdot n + \frac{b_1 - b_2}{a_1 + a_2} \le \sigma \le \frac{n}{2}\\
f_2(p_2) \le f_2(p_1) &\;\Rightarrow\; \frac{a_2}{a_1 + a_2} \cdot n + \frac{b_2 - b_1}{a_1 + a_2} \le \sigma \le \frac{n}{2}
\end{align*}
the expected time for the \muea with genotypic clearing, $\kappa\in\mathbb{N}$, $\mu \ge \kappa \cdot \frac{\sigma n - 2\sigma + 2}{n-2\sigma+2}$, 
$\mu\leq\poly{n}$ and clearing radius~$\sigma$ finding all global optima of $f_2$ is $O(\mu n \log{n})$.
\end{theorem}
Note that in case $f_2(p_1) \neq f_2(p_2)$ there is only one global optimum: the fitter of the two peaks. Then the respective condition (where the left-hand side inequality is true) implies that the basin of attraction of the less fit peak must be bounded by $n/2$. If this condition is not satisfied, the function is deceptive as the majority of the search space leads towards a non-optimal local optimum. 

In case $f_2(p_1) = f_2(p_2)$ both peaks are global optima and the conditions require that both basins of attraction have size $n/2$:
\[
\sigma = \frac{a_1}{a_1 + a_2} \cdot n + \frac{b_1 - b_2}{a_1 + a_2} = \frac{a_2}{a_1 + a_2} \cdot n + \frac{b_2 - b_1}{a_1 + a_2} = \frac{n}{2}.
\]
\begin{proof}[Proof of Theorem~\ref{the:fam_f2}]
The proof is similar to the proof of Theorem~\ref{the:few-niches-twomax}. Assume without loss of generality that $f_2(p_1) \le f_2(p_2)$. Using the same arguments as in said proof (with straightforward changes to the fitness-level calculations), the \muea finds one peak in expected time $O(\mu n \log n)$. If this is $p_1$, the \muea still needs to find~$p_2$. By the same arguments as in the proof of Theorem~\ref{the:few-niches-twomax}, the \muea's population will contain $\kappa$ copies of $p_1$ in expected time $O(\mu \log n)$. Applying Lemma~\ref{lem:time-to-clearing-radius} with $\mindist=\sigma$ yields that the expected time to find a search point~$x$ with Hamming distance at least $\sigma$ to~$p_1$ is $O(\mu n \log n)$. Since $\frac{a_1}{a_1 + a_2} \cdot n + \frac{b_1 - b_2}{a_1 + a_2} \le \sigma$, by Theorem~2 in~\cite{Jansen2016}, $x$ is outside the basin of attraction of~$p_1$. As it is also a winner in a new niche, this new niche will never be removed, and $p_2$ can be reached by hill climbing on a \onemax-like slope from~$x$. By previous arguments, $p_2$ will then be found in expected time $O(\mu n \log n)$.
\end{proof}

As a final remark, the analysis has shown that it is possible to escape of the basin of attraction of the higher peak with $\B \leq n/2$, this does not mean that the analysis cannot be applied to $\B \geq n/2$. We need to remember that the current investigation considers a distance $d\leq n/2$ because any distance larger than $n/2$ may lead to a exponential expected time in $n$ for reaching a distance of at least $d$ from $\winner$. One way to avoid this limitation is by dividing the distance $d$ into several niches by setting the parameter $\sigma\leq n/2$ properly. In this analysis we just considered the population dynamics and its ability of escaping a basin of attraction of at most $n/2$ or escaping from a niche with radius at most $n/2$ but it may be possible to generalise the population dynamics for more than 2 niches with sizes $\leq n/2$ by changing/adapting our definition of the potential function. For the time being we rely on experiments in Section~\ref{sec:escdifbasatt} to show that the population can jump from niches with $\sigma\leq n/2$ allowing to find both optimum in different variants of \twomax from the classes of example functions and leave the generalisation of the population dynamics for future theoretical work.

\section{Experiments}
\label{sec:exp}

The experimental approach is focused on the analysis of the \muea and is divided in 3 experimental frameworks. Section~\ref{sec:empangen} is focused on an empirical analysis for the general behaviour of the algorithm, the relationship between the parameters $\sigma$, $\kappa$, and $\mu$, and how these parameters can be set. The main objective is to compare our asymptotic theoretical results with empirical data for concrete parameter values.

For the second empirical analysis (Section~\ref{sec:empanpopsize}), we focus our attention on the population size for small ($n=30$) and large ($n=100$) problem sizes. The objective is to observe whether smaller population sizes than $\mu = \kappa n^2/4$ are capable of optimising \twomax and compare if the quadratic dependence on $n$ is an artefact of our approach. Also we compare two different forms of initialising the population: the standard uniform random initialisation against a biased initialisation where the whole population is initialised with copies of one peak ($0^n$ for \twomax). Biased initialisation is used in order to observe how \clearing is able to escape from a local optimum and how fast it is compared to a random initialisation.

Finally, for the third analysis (Section~\ref{sec:escdifbasatt}), we show that it is possible to escape from different basin of attractions for weighted peak functions with two peaks in cases where the two peaks are not complementary, but have different Hamming distances. 

\subsection{General Behaviour}
\label{sec:empangen}

We are interested in observing if the \muea with \clearing is able to find both optima on \twomax, so we consider exponentially increasing population sizes~$\mu = \{2,4,8,\ldots,1024\}$ for just one size of $n=30$ and perform $100$ runs with different settings of parameters $\sigma$ and $\kappa$, so for this experimental framework, we have defined $\sigma = \{1, 2, \sqrt[]{n}, n/2\}$, $\kappa = \{1, \sqrt[] {\mu}, \mu/2, \mu\}$ with phenotypic distance since it has been proven that this distance metric works for both cases, small and large niches (when the genotypic distance is used it will be explicitly mentioned).

Since we are interested in proving how good/bad \clearing is, we define the following outcomes and stopping criteria for each run. \emph{Success}, the run is stopped if the population contains both $0^n$ and $1^n$ in the population. \emph{Failure}, once the run has reached 1 million of generations and one of the two (or both) optima are not contained in the population. All the results are shown in Table~\ref{tab:expgenres}.

\begin{table}[ht]
\centering
\caption{Success rate measured among $100$ runs for the \muea with phenotypic clearing on \twomax for $n=30$ for the different parameters \clearingradius $\sigma$, \nichecapacity~$\kappa$ and \emph{population size} $\mu$.}
\label{tab:expgenres}
\begin{tabular}{ccccccccccc}  
\toprule
\multicolumn{11}{c}{\boldmath$\sigma = 1$}\\
\midrule
\multirow{2}{*}{$\kappa$} & \multicolumn{10}{c}{$\mu$} \\
\cmidrule(l){2-11}
 & $2$ & $4$ & $8$ & $16$ & $32$ & $64$ & $128$ & $256$ & $512$ & $1024$ \\
\midrule
$1$ & $0.0$ & $0.05$ & $0.96$ & $1.0$ & $1.0$ & $1.0$ & $1.0$ & $1.0$ & $1.0$ & $1.0$ \\
$\sqrt[]{\mu}$ & $0.0$ & $0.0$ & $0.0$ & $0.38$ & $0.89$ & $1.0$ & $1.0$ & $1.0$ & $1.0$ & $1.0$ \\
$\mu/2$ & $0.0$ & $0.0$ & $0.0$ & $0.0$ & $0.04$ & $0.20$ & $0.42$ & $0.77$ & $0.97$ & $0.98$ \\
$\mu$ & $0.0$ & $0.0$ & $0.0$ & $0.0$ & $0.01$ & $0.13$ & $0.24$ & $0.55$ & $0.75$ & $0.94$ \\
\toprule
\multicolumn{11}{c}{\boldmath$\sigma = 2$}\\
\midrule
\multirow{2}{*}{$\kappa$} & \multicolumn{10}{c}{$\mu$} \\
\cmidrule(l){2-11}
 & $2$ & $4$ & $8$ & $16$ & $32$ & $64$ & $128$ & $256$ & $512$ & $1024$ \\
\midrule
$1$ & $0.02$ & $0.88$ & $1.0$ & $1.0$ & $1.0$ & $1.0$ & $1.0$ & $1.0$ & $1.0$ & $1.0$ \\
$\sqrt[]{\mu}$ & $0.01$ & $0.03$ & $0.55$ & $0.99$ & $1.0$ & $1.0$ & $1.0$ & $1.0$ & $1.0$ & $1.0$ \\
$\mu/2$ & $0.0$ & $0.0$ & $0.0$ & $0.0$ & $0.07$ & $0.18$ & $0.48$ & $0.67$ & $0.93$ & $0.99$ \\
$\mu$ & $0.0$ & $0.0$ & $0.0$ & $0.0$ & $0.00$ & $0.04$ & $0.25$ & $0.60$ & $0.80$ & $0.97$ \\
\toprule
\multicolumn{11}{c}{\boldmath$\sigma = \sqrt[]{n}$}\\
\midrule
\multirow{2}{*}{$\kappa$} & \multicolumn{10}{c}{$\mu$} \\
\cmidrule(l){2-11}
 & $2$ & $4$ & $8$ & $16$ & $32$ & $64$ & $128$ & $256$ & $512$ & $1024$ \\
\midrule
$1$ & $0.33$ & $1.0$ & $1.0$ & $1.0$ & $1.0$ & $1.0$ & $1.0$ & $1.0$ & $1.0$ & $1.0$ \\
$\sqrt[]{\mu}$ & $0.35$ & $0.67$ & $0.97$ & $1.0$ & $1.0$ & $1.0$ & $1.0$ & $1.0$ & $1.0$ & $1.0$ \\
$\mu/2$ & $0.40$ & $0.78$ & $0.95$ & $1.0$ & $1.0$ & $1.0$ & $1.0$ & $1.0$ & $1.0$ & $1.0$ \\
$\mu$ & $0.0$ & $0.0$ & $0.0$ & $0.0$ & $0.0$ & $0.07$ & $0.28$ & $0.50$ & $0.80$ & $0.93$ \\
\toprule
\multicolumn{11}{c}{\boldmath$\sigma = n/2$}\\
\midrule
\multirow{2}{*}{$\kappa$} & \multicolumn{10}{c}{$\mu$} \\
\cmidrule(l){2-11}
 & $2$ & $4$ & $8$ & $16$ & $32$ & $64$ & $128$ & $256$ & $512$ & $1024$ \\
\midrule
$1$ & $1.0$ & $1.0$ & $1.0$ & $1.0$ & $1.0$ & $1.0$ & $1.0$ & $1.0$ & $1.0$ & $1.0$ \\
$\sqrt[]{\mu}$ & $1.0$ & $1.0$ & $1.0$ & $1.0$ & $1.0$ & $1.0$ & $1.0$ & $1.0$ & $1.0$ & $1.0$ \\
$\mu/2$ & $1.0$ & $1.0$ & $1.0$ & $1.0$ & $1.0$ & $1.0$ & $1.0$ & $1.0$ & $1.0$ & $1.0$ \\
$\mu$ & $0.0$ & $0.0$ & $0.0$ & $0.0$ & $0.02$ & $0.12$ & $0.29$ & $0.60$ & $0.81$ & $0.98$ \\
\bottomrule
\end{tabular}
\end{table}

For small values of $\sigma = \{1, 2\}$ and $\kappa = 1$, with sufficiently many individuals~$\mu=(n/2+1)\cdot\kappa$, every individual can create its own niche, and since only one individual is allowed to be the winner, the individuals are spread in the search space reaching both optima with $1.0$ success rate. In this scenario, since we are allowing sufficiently many individuals in the population, individuals can be initialised in any of both branches, climbing down the branch reaching the opposite branch and reaching the other extreme optima as shown in Figure~\ref{fig:scenario1} (in this case we only show the behaviour of the population with $\mu = \{8,16,32\}$, with $\mu\geq 8$ have the same behaviour).

\begin{figure*}[ht]
    \centering
    \begin{subfigure}[t]{0.31\textwidth}
    	\centering
        \resizebox{\linewidth}{!}{
        	\begin{tikzpicture}[domain=0:30,xscale=0.15,yscale=0.25, scale=1, every 
            shadow/.style={shadow xshift=0.0mm, shadow yshift=0.4mm}]
        		\tikzstyle{helpline}=[black,thick];
        		\tikzstyle{function}=[black,very thick];
                \tikzstyle{optimum}=[red,very thick];
                \tikzstyle{winner}=[blue,very thick];
                \tikzstyle{cleared}=[green,very thick];
        		\draw[black!40,line width=0.2pt,xstep=1,ystep=1] (0,0) grid (30.5,15.5);
        		\draw[function] (0,15) -- (15,0) -- (30,15);
        		\draw[helpline, thin, -triangle 45] (0,0) -- (0,17);
        		\draw[helpline, thin, -triangle 45] (0,0) -- (32,0) node[right] {\small $\ones{x}$};
        		\draw[helpline] (0,0) -- (0,16);
        		\draw[helpline] (0,0) -- (31,0);
        		\draw[helpline] (0,0) -- (-1,0) node[left] {\small $n/2$};
        		\draw[helpline] (0,0) -- (0,-1) node[below] {\small $0$};
        		\draw[helpline] (0,15) -- (-1,15) node[left] {\small $n$};
        		\draw[helpline] (15,0) -- (15,-1) node[below] {\small $n/2$};
        		\draw[helpline] (30,0) -- (30,-1) node[below] {\small $n$};
                \filldraw[optimum] (0,15) circle (6pt);
                \foreach \x/\y in {1/14,2/13,3/12}
                	\filldraw[winner] (\x,\y) circle (6pt);
                \filldraw[optimum] (30,15) circle (6pt);
                \foreach \x/\y in {29/14,28/13,27/12}
                	\filldraw[winner] (\x,\y) circle (6pt);
        	\end{tikzpicture}
		}
        \caption{$\mu = 8$}
        \label{fig:mu8}
    \end{subfigure}
    ~
    \begin{subfigure}[t]{0.31\textwidth}
    	\centering
        \resizebox{\linewidth}{!}{
       \begin{tikzpicture}[domain=0:30,xscale=0.15,yscale=0.25, scale=1, every 
       shadow/.style={shadow xshift=0.0mm, shadow yshift=0.4mm}]
				\tikzstyle{helpline}=[black,thick];
				\tikzstyle{function}=[black,very thick];
                \tikzstyle{optimum}=[red,very thick];
                \tikzstyle{winner}=[blue,very thick];
                \tikzstyle{cleared}=[green,very thick];
                \draw[black!40,line width=0.2pt,xstep=1,ystep=1] (0,0) grid (30.5,15.5);
                \draw[function] (0,15) -- (15,0) -- (30,15);
                \draw[helpline, thin, -triangle 45] (0,0) -- (0,17);
                \draw[helpline, thin, -triangle 45] (0,0) -- (32,0) node[right] 
                {\small $\ones{x}$};
                \draw[helpline] (0,0) -- (0,16);
                \draw[helpline] (0,0) -- (31,0);
                \draw[helpline] (0,0) -- (-1,0) node[left] {\small $n/2$};
                \draw[helpline] (0,0) -- (0,-1) node[below] {\small $0$};
                \draw[helpline] (0,15) -- (-1,15) node[left] {\small $n$};
                \draw[helpline] (15,0) -- (15,-1) node[below] {\small $n/2$};
                \draw[helpline] (30,0) -- (30,-1) node[below] {\small $n$};
                \filldraw[optimum] (0,15) circle (6pt);
                \foreach \x/\y in {1/14,2/13,3/12,4/11,5/10,6/9,7/8}
                	\filldraw[winner] (\x,\y) circle (6pt);
                \filldraw[optimum] (30,15) circle (6pt);
                \foreach \x/\y in {29/14,28/13,27/12,26/11,25/10,24/9,23/8}
                	\filldraw[winner] (\x,\y) circle (6pt);
			\end{tikzpicture}
		}
        \caption{$\mu = 16$}
        \label{fig:mu16}
    \end{subfigure}
    ~
    \begin{subfigure}[t]{0.31\textwidth}
    \centering
        \resizebox{\linewidth}{!}{
    		\begin{tikzpicture}[domain=0:30,xscale=0.15,yscale=0.25, scale=1, every 
            shadow/.style={shadow xshift=0.0mm, shadow yshift=0.4mm}]
                \tikzstyle{helpline}=[black,thick];
                \tikzstyle{function}=[black,very thick];
                \tikzstyle{optimum}=[red,very thick];
                \tikzstyle{winner}=[blue,very thick];
                \tikzstyle{cleared}=[green,very thick];
                \draw[black!40,line width=0.2pt,xstep=1,ystep=1] (0,0) grid (30.5,15.5);
                \draw[function] (0,15) -- (15,0) -- (30,15);
                \draw[helpline, thin, -triangle 45] (0,0) -- (0,17);
                \draw[helpline, thin, -triangle 45] (0,0) -- (32,0) node[right] {\small 
                $\ones{x}$};
                \draw[helpline] (0,0) -- (0,16);
                \draw[helpline] (0,0) -- (31,0);
                \draw[helpline] (0,0) -- (-1,0) node[left] {\small $n/2$};
                \draw[helpline] (0,0) -- (0,-1) node[below] {\small $0$};
                \draw[helpline] (0,15) -- (-1,15) node[left] {\small $n$};
                \draw[helpline] (15,0) -- (15,-1) node[below] {\small $n/2$};
                \draw[helpline] (30,0) -- (30,-1) node[below] {\small $n$};
                \filldraw[optimum] (0,15) circle (6pt);
                \foreach \x/\y in {1/14,2/13,3/12,4/11,5/10,6/9,7/8,8/7,9/6,10/5,11/4,12/3,13/2,14/1}
                	\filldraw[winner] (\x,\y) circle (6pt);
                \filldraw[optimum] (30,15) circle (6pt);
                \foreach \x/\y in 
                {29/14,28/13,27/12,26/11,25/10,24/9,23/8,22/7,21/6,20/5,19/4,18/3,17/2,16/1,15/0}
                	\filldraw[winner] (\x,\y) circle (6pt);
			\end{tikzpicture}
		}
        \caption{$\mu = 32$}
        \label{fig:mu32}
    \end{subfigure}
    \caption{Snapshot of a typical population at the time both optima were reached, showing the spread of individuals in branches of \twomax for $n=30$, $\sigma = 1$ and $\kappa = 1$. Where the red (extreme) points represent optimal individuals, blue points represent niche winners. The rows on the grid represents the fitness value of an individual and its position on \twomax and the vertical lines represent the partitioned search space (niches) created by the parameter $\sigma$.}
    \label{fig:scenario1}
\end{figure*}
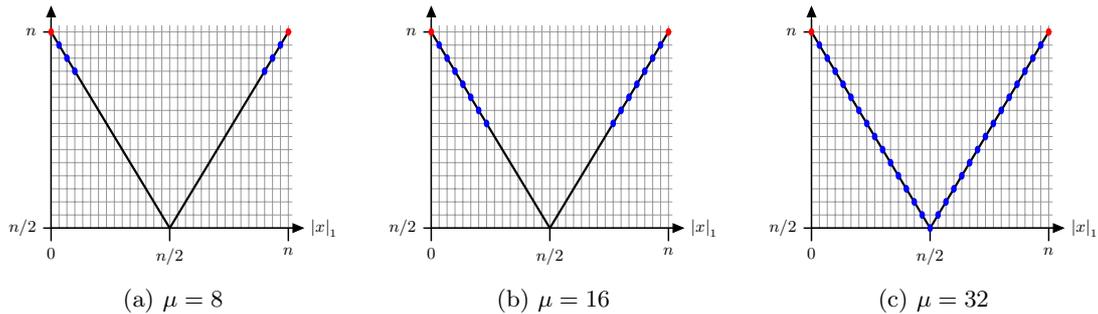

The previous experiment set-up confirms what is mentioned in Section~\ref{sec:phenotypic_clearing}: with a small \clearingradius, \nichecapacity and large enough \emph{population size}, the algorithm is able to exhaustively explore the search space without losing the progress reached so far. The population size provides enough pressure to optimise \twomax. In this scenario the small differences between individuals allow the algorithm to discriminate between the two branches or optima, this forces to have individuals on both branches or occupy all niches supporting the statement of Theorem~\ref{the:optunitation}. Individuals with the same phenotype may have a large Hamming distance, creating winners with the same fitness (as proved by Corollary \ref{cor:clearing-small-niches-genotypic} in Section~\ref{sec:genotypic_clearing}).

It is mentioned in \cite{Petrowski1996} that while the value of the \nichecapacity $\kappa > \mu/2$ approaches the size of the population, the clearing effect vanishes and the search becomes a standard EA. This effect is verified in the present experimental approach. For a large $\kappa$ and $\sigma=\{1,2\}$, one branch takes over, removing the individuals on the other branch reducing the performance of the algorithm. In order to avoid this, it is necessary to define $\mu\geq (n+1)\cdot \kappa$ to occupy all the winners slots and create new winners in other niches (Theorem~\ref{the:optunitation}) or increase the \clearingradius to $\sqrt{n} \leq \sigma \leq n/2$ in order to let more individuals participate in the niche. A reduced \nichecapacity $1 \leq \kappa \leq \sqrt[]{\mu}$ seems to have a better effect exploring both branches. 

Now that theory and practice have shown that a small \clearingradius, \nichecapacity and large enough population size $\mu$ are able to optimise \twomax, and in order to avoid the take over of a certain branch due to a large \nichecapacity it is necessary to either have: (1) a large enough population, or (2) to increase the \clearingradius. For (1) we already have defined and proved that one way to overcome this scenario is to define $\mu$ according to Theorem~\ref{the:optunitation}.

In the case of large niches (2), with $\sigma = \{\sqrt[]{n},n/2\}$ it is possible to divide the search space in fewer niches.  Here the individuals have the opportunity to move, change inside the niche, reach other niches allowing the movement between branches, reaching the opposite optimum. Since it is possible to reach other niches, defining the \nichecapacity $\kappa=\sqrt[]{\mu}$ will allow to have more winners in each niche but will still allow to move inside the niche.

For example, with $\sigma = \sqrt[] {n}$, $\kappa=\sqrt[]{\mu}$ and $\mu\geq 8$, the algorithm is able to reach both optima with at least $0.97$ success rate. In Figure \ref{fig:scenario2} the effect of $\kappa$ can be seen with sufficiently many individuals. With restrictive \nichecapacity (Figure~\ref{fig:kappa1}), the population is scattered in the search space but when the \nichecapacity is increased, the spread is reduced as we allow more individuals to be part of each niche (Figure~\ref{fig:kappasqrtmu} and \ref{fig:kappamuhalf}). This behaviour can be generalised and is more evident for larger values of $\mu$.

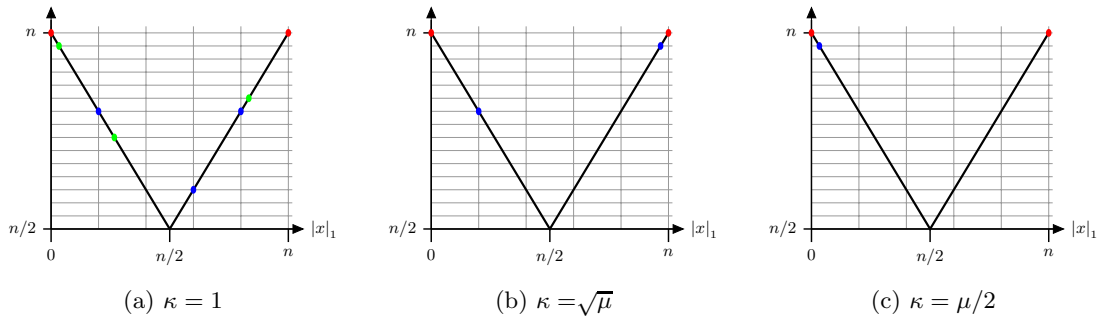
\begin{figure*}[ht]
    \centering
    \begin{subfigure}[t]{0.31\textwidth}
        \centering
        \resizebox{\linewidth}{!}{
        	\begin{tikzpicture}[domain=0:30,xscale=0.15,yscale=0.25, scale=1, 
            every shadow/.style={shadow xshift=0.0mm, shadow yshift=0.4mm}]
                \tikzstyle{helpline}=[black,thick];
        		\tikzstyle{function}=[black,very thick];
                \tikzstyle{optimum}=[red,very thick];
                \tikzstyle{winner}=[blue,very thick];
                \tikzstyle{cleared}=[green,very thick];
        		\draw[black!40,line width=0.2pt,xstep=6,ystep=1] (0,0) grid (30.5,15.5);
                \draw[function] (0,15) -- (15,0) -- (30,15);
        		\draw[helpline, thin, -triangle 45] (0,0) -- (0,17);
        		\draw[helpline, thin, -triangle 45] (0,0) -- (32,0) node[right] {\small     
                $\ones{x}$};
        		\draw[helpline] (0,0) -- (0,16);
        		\draw[helpline] (0,0) -- (31,0);
        		\draw[helpline] (0,0) -- (-1,0) node[left] {\small $n/2$};
        		\draw[helpline] (0,0) -- (0,-1) node[below] {\small $0$};
        		\draw[helpline] (0,15) -- (-1,15) node[left] {\small $n$};
        		\draw[helpline] (15,0) -- (15,-1) node[below] {\small $n/2$};
        		\draw[helpline] (30,0) -- (30,-1) node[below] {\small $n$};
                \filldraw[optimum] (0,15) circle (6pt);
                \filldraw[cleared] (1,14) circle (6pt);
                \filldraw[winner] (6,9) circle (6pt);
                \filldraw[cleared] (8,7) circle (6pt);
                \filldraw[optimum] (30,15) circle (6pt);
                \filldraw[cleared] (25,10) circle (6pt);
                \filldraw[winner] (24,9) circle (6pt);
                \filldraw[winner] (18,3) circle (6pt);
        	\end{tikzpicture}
		}
        \caption{$\kappa = 1$}
        \label{fig:kappa1}
    \end{subfigure}
    ~
    \begin{subfigure}[t]{0.31\textwidth}
        \centering
        \resizebox{\linewidth}{!}{
        	\begin{tikzpicture}[domain=0:30,xscale=0.15,yscale=0.25, scale=1, 
            every shadow/.style={shadow xshift=0.0mm, shadow yshift=0.4mm}]
                \tikzstyle{helpline}=[black,thick];
        		\tikzstyle{function}=[black,very thick];
                \tikzstyle{optimum}=[red,very thick];
                \tikzstyle{winner}=[blue,very thick];
                \tikzstyle{cleared}=[green,very thick];
        		\draw[black!40,line width=0.2pt,xstep=6,ystep=1] (0,0) grid (30.5,15.5);
                \draw[function] (0,15) -- (15,0) -- (30,15);
        		\draw[helpline, thin, -triangle 45] (0,0) -- (0,17);
        		\draw[helpline, thin, -triangle 45] (0,0) -- (32,0) node[right] {\small     
                $\ones{x}$};
        		\draw[helpline] (0,0) -- (0,16);
        		\draw[helpline] (0,0) -- (31,0);
        		\draw[helpline] (0,0) -- (-1,0) node[left] {\small $n/2$};
        		\draw[helpline] (0,0) -- (0,-1) node[below] {\small $0$};
        		\draw[helpline] (0,15) -- (-1,15) node[left] {\small $n$};
        		\draw[helpline] (15,0) -- (15,-1) node[below] {\small $n/2$};
        		\draw[helpline] (30,0) -- (30,-1) node[below] {\small $n$};
                \filldraw[optimum] (0,15) circle (6pt);
                \filldraw[winner] (6,9) circle (6pt);
                \filldraw[optimum] (30,15) circle (6pt);
                \filldraw[winner] (29,14) circle (6pt);
        	\end{tikzpicture}
		}
        \caption{$\kappa = \sqrt[]{\mu}$}
        \label{fig:kappasqrtmu}
    \end{subfigure}
    ~
    \begin{subfigure}[t]{0.31\textwidth}
        \centering
        \resizebox{\linewidth}{!}{
        	\begin{tikzpicture}[domain=0:30,xscale=0.15,yscale=0.25, scale=1, 
            every shadow/.style={shadow xshift=0.0mm, shadow yshift=0.4mm}]
                \tikzstyle{helpline}=[black,thick];
        		\tikzstyle{function}=[black,very thick];
                \tikzstyle{optimum}=[red,very thick];
                \tikzstyle{winner}=[blue,very thick];
                \tikzstyle{cleared}=[green,very thick];
        		\draw[black!40,line width=0.2pt,xstep=6,ystep=1] (0,0) grid (30.5,15.5);
                \draw[function] (0,15) -- (15,0) -- (30,15);
        		\draw[helpline, thin, -triangle 45] (0,0) -- (0,17);
        		\draw[helpline, thin, -triangle 45] (0,0) -- (32,0) node[right] {\small     
                $\ones{x}$};
        		\draw[helpline] (0,0) -- (0,16);
        		\draw[helpline] (0,0) -- (31,0);
        		\draw[helpline] (0,0) -- (-1,0) node[left] {\small $n/2$};
        		\draw[helpline] (0,0) -- (0,-1) node[below] {\small $0$};
        		\draw[helpline] (0,15) -- (-1,15) node[left] {\small $n$};
        		\draw[helpline] (15,0) -- (15,-1) node[below] {\small $n/2$};
        		\draw[helpline] (30,0) -- (30,-1) node[below] {\small $n$};
                \filldraw[optimum] (0,15) circle (6pt);
                \filldraw[winner] (1,14) circle (6pt);
                \filldraw[optimum] (30,15) circle (6pt);
        	\end{tikzpicture}
		}
        \caption{$\kappa = \mu/2$}
        \label{fig:kappamuhalf}
    \end{subfigure}
    \caption{Snapshot of a typical population at the time both optima were reached, showing the spread of individuals in branches of \twomax for $n=30$, $\sigma = \sqrt[]{n}$ and $\mu = 8$. Where the red (extreme) points represent optimal individuals, blue points represent niche winners, and the green points represent cleared individuals. The rows on the grid represents the fitness value of an individual and its position on \twomax and the columns represent the partitioned search space (niches) created by the parameter $\sigma$.}
    \label{fig:scenario2}
\end{figure*}

The theoretical results described in Section~\ref{sec:large_niches} are confirmed by the previous experimental results, a large enough population is necessary in order to fill the positions of the winner $\winner$ with $\kappa$ winners, then force those $\kappa$ winners with the rest of the population to be subject to a random walk, where it is just necessary for at least one individual to reach the next niche as mentioned in Section~\ref{sec:pop-dyn}. In the case of \twomax, it is after the repetition of moving, climbing down through different niches for a certain period of time when some individual is able to reach both optima as mentioned in Theorem~\ref{the:few-niches-twomax} and confirmed by the experiments.

Now we have defined the conditions where the algorithm is able to optimise \twomax, we can set-up the parameters in a more informed/smart way. With $\mu \geq 2$ it is possible to optimise \twomax if $\sigma$ and $\kappa$ are chosen appropriately. For example, with $\sigma=n/2$ (as the minimum distance required to distinguish between one branch and the other), $\kappa=\{1,\sqrt[]{\mu}, \mu/2\}$ and $\mu \geq 2$, the algorithm is able to optimise \twomax because there is always an individual moving around that is able to reach a new niche (Figure~\ref{fig:scenario3}), and finally achieve $1.0$ success rate.

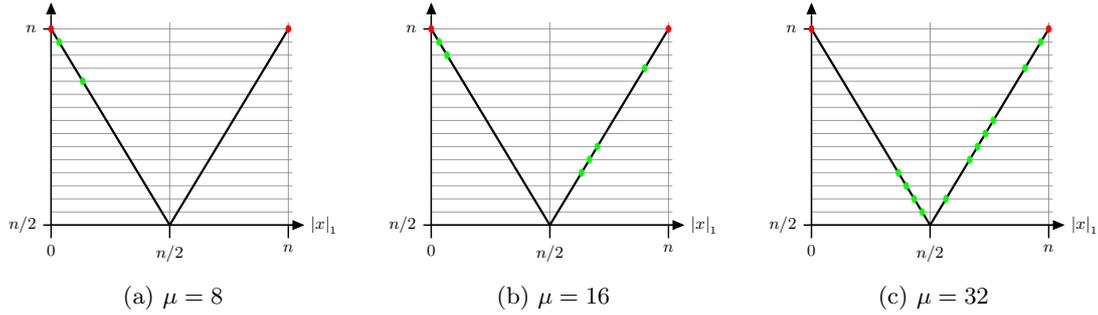
\begin{figure*}[ht]
    \centering
    \begin{subfigure}[t]{0.31\textwidth}
        \centering
        \resizebox{\linewidth}{!}{
        	\begin{tikzpicture}[domain=0:30,xscale=0.15,yscale=0.25, scale=1, 
            every shadow/.style={shadow xshift=0.0mm, shadow yshift=0.4mm}]
                \tikzstyle{helpline}=[black,thick];
        		\tikzstyle{function}=[black,very thick];
                \tikzstyle{optimum}=[red,very thick];
                \tikzstyle{winner}=[blue,very thick];
                \tikzstyle{cleared}=[green,very thick];
        		\draw[black!40,line width=0.2pt,xstep=15,ystep=1] (0,0) grid (30.5,15.5);
                \draw[function] (0,15) -- (15,0) -- (30,15);
        		\draw[helpline, thin, -triangle 45] (0,0) -- (0,17);
        		\draw[helpline, thin, -triangle 45] (0,0) -- (32,0) node[right] {\small     
                $\ones{x}$};
        		\draw[helpline] (0,0) -- (0,16);
        		\draw[helpline] (0,0) -- (31,0);
        		\draw[helpline] (0,0) -- (-1,0) node[left] {\small $n/2$};
        		\draw[helpline] (0,0) -- (0,-1) node[below] {\small $0$};
        		\draw[helpline] (0,15) -- (-1,15) node[left] {\small $n$};
        		\draw[helpline] (15,0) -- (15,-1) node[below] {\small $n/2$};
        		\draw[helpline] (30,0) -- (30,-1) node[below] {\small $n$};
                \filldraw[optimum] (0,15) circle (6pt);
                \filldraw[cleared] (1,14) circle (6pt);
                \filldraw[cleared] (4,11) circle (6pt);
                \filldraw[optimum] (30,15) circle (6pt);
        	\end{tikzpicture}
		}
        \caption{$\mu = 8$}
        \label{fig:mu81}
    \end{subfigure}
    ~
    \begin{subfigure}[t]{0.31\textwidth}
        \centering
        \resizebox{\linewidth}{!}{
        	\begin{tikzpicture}[domain=0:30,xscale=0.15,yscale=0.25, scale=1, 
            every shadow/.style={shadow xshift=0.0mm, shadow yshift=0.4mm}]
                \tikzstyle{helpline}=[black,thick];
        		\tikzstyle{function}=[black,very thick];
                \tikzstyle{optimum}=[red,very thick];
                \tikzstyle{winner}=[blue,very thick];
                \tikzstyle{cleared}=[green,very thick];
        		\draw[black!40,line width=0.2pt,xstep=15,ystep=1] (0,0) grid (30.5,15.5);
                \draw[function] (0,15) -- (15,0) -- (30,15);
        		\draw[helpline, thin, -triangle 45] (0,0) -- (0,17);
        		\draw[helpline, thin, -triangle 45] (0,0) -- (32,0) node[right] {\small     
                $\ones{x}$};
        		\draw[helpline] (0,0) -- (0,16);
        		\draw[helpline] (0,0) -- (31,0);
        		\draw[helpline] (0,0) -- (-1,0) node[left] {\small $n/2$};
        		\draw[helpline] (0,0) -- (0,-1) node[below] {\small $0$};
        		\draw[helpline] (0,15) -- (-1,15) node[left] {\small $n$};
        		\draw[helpline] (15,0) -- (15,-1) node[below] {\small $n/2$};
        		\draw[helpline] (30,0) -- (30,-1) node[below] {\small $n$};
                \filldraw[optimum] (0,15) circle (6pt);
                \filldraw[cleared] (1,14) circle (6pt);
                \filldraw[cleared] (2,13) circle (6pt);
                \filldraw[optimum] (30,15) circle (6pt);
                \filldraw[cleared] (27,12) circle (6pt);
                \filldraw[cleared] (21,6) circle (6pt);
                \filldraw[cleared] (20,5) circle (6pt);
                \filldraw[cleared] (19,4) circle (6pt);
        	\end{tikzpicture}
		}
        \caption{$\mu = 16$}
        \label{fig:mu161}
    \end{subfigure}
    ~
    \begin{subfigure}[t]{0.31\textwidth}
        \centering
        \resizebox{\linewidth}{!}{
        	\begin{tikzpicture}[domain=0:30,xscale=0.15,yscale=0.25, scale=1, 
            every shadow/.style={shadow xshift=0.0mm, shadow yshift=0.4mm}]
                \tikzstyle{helpline}=[black,thick];
        		\tikzstyle{function}=[black,very thick];
                \tikzstyle{optimum}=[red,very thick];
                \tikzstyle{winner}=[blue,very thick];
                \tikzstyle{cleared}=[green,very thick];
        		\draw[black!40,line width=0.2pt,xstep=15,ystep=1] (0,0) grid (30.5,15.5);                
                \draw[function] (0,15) -- (15,0) -- (30,15);
        		\draw[helpline, thin, -triangle 45] (0,0) -- (0,17);
        		\draw[helpline, thin, -triangle 45] (0,0) -- (32,0) node[right] {\small     
                $\ones{x}$};
        		\draw[helpline] (0,0) -- (0,16);
        		\draw[helpline] (0,0) -- (31,0);
        		\draw[helpline] (0,0) -- (-1,0) node[left] {\small $n/2$};
        		\draw[helpline] (0,0) -- (0,-1) node[below] {\small $0$};
        		\draw[helpline] (0,15) -- (-1,15) node[left] {\small $n$};
        		\draw[helpline] (15,0) -- (15,-1) node[below] {\small $n/2$};
        		\draw[helpline] (30,0) -- (30,-1) node[below] {\small $n$};
                \filldraw[optimum] (0,15) circle (6pt);
                \filldraw[cleared] (11,4) circle (6pt);
                \filldraw[cleared] (12,3) circle (6pt);
                \filldraw[cleared] (13,2) circle (6pt);
                \filldraw[cleared] (14,1) circle (6pt);
                \filldraw[optimum] (30,15) circle (6pt);
				\filldraw[cleared] (29,14) circle (6pt);
                \filldraw[cleared] (27,12) circle (6pt);
                \filldraw[cleared] (23,8) circle (6pt);
                \filldraw[cleared] (22,7) circle (6pt);
                \filldraw[cleared] (21,6) circle (6pt);
                \filldraw[cleared] (20,5) circle (6pt);
                \filldraw[cleared] (17,2) circle (6pt);
        	\end{tikzpicture}
		}
        \caption{$\mu = 32$}
        \label{fig:mu321}
    \end{subfigure}
    \caption{Snapshot of a typical population at the time both optima were reached, showing the spread of individuals in branches of \twomax for $n=30$, $\sigma = n/2$ and $\kappa = 1$. Where the red (extreme) points represent optimal individuals, blue points represent niche winners, and the green points represent cleared individuals. The rows on the grid represents the fitness value of an individual and its position on \twomax and the columns represent the partitioned search space (niches) created by the parameter $\sigma$.}
    \label{fig:scenario3}
\end{figure*}

\subsection{Population Size}
\label{sec:empanpopsize}

In this section we address the limitation of Theorem~\ref{the:few-niches-twomax} related to the steep requirement of the population size: $\mu\geq\kappa n^2/4$. As observed from Section~\ref{sec:empangen}, experiments suggest that a smaller population size is able to optimise \twomax. So for the analysis of the population size we have considered the population size $2\leq\mu\leq \kappa n^2/4$ in order to observe what is the minimum population size below the threshold $\kappa n^2/4$ able to optimise \twomax. With $\sigma=n/2$, and $\kappa=1$ for $n=\{30,100\}$ with phenotypic clearing, we report the average of generations of $100$ runs, the run is stopped if both optima have been found or the algorithm has reached a maximum of 1 million generations, this is enough time for the algorithms to converge on one or both optima. 

Figure~\ref{fig:n30mulogscale} shows the average number of generations among $100$ runs with $n=30$. Even for $\mu=2$ the average runtime is below the $1$ million threshold, hence some of the runs were able to find both optima on \twomax in fewer than $1$ million generations. The reason for this high average runtime is because once both individuals have reached one optimum, it will be one winner, and one loser subjected to a random walk until it gets replaced by an offspring of the winner. This process will continue until the loser reaches a Hamming distance of $n/2$ from the optimum to escape of the basin of attraction. Once this is achieved, it is just necessary for the individual to climb the other branch.

\begin{figure*}[ht]
    \centering
    \begin{subfigure}[t]{0.45\textwidth}
        \centering
        \resizebox{\linewidth}{!}{
        	\begin{tikzpicture}	
    			\begin{axis}[
                    width=7cm,
        			height=6cm,
                    xlabel=$\mu$ (logscale),
        			xtick={1,2,3,4,5,6,7,8},
                    xticklabels={2,4,8,16,32,64,128,256},
        			ylabel=$t$,
        			every axis y label/.style={rotate=0, black, at={(-0.13,0.5)},},
					]
						\addplot table[x=mu, y=trn30]{\averagethirtylog};
                    	\addplot table[x=mu, y=t0n30]{\averagethirtylog};
    			\end{axis}
		\end{tikzpicture}
		}
        \caption{$n=30$}
        \label{fig:n30mulogscale}
    \end{subfigure}
    \begin{subfigure}[t]{0.45\textwidth}
        \centering
        \resizebox{\linewidth}{!}{
        	\begin{tikzpicture}	
        	\begin{axis}[
        		width=7cm,
        		height=6cm,
        		xlabel=$\mu$ (logscale),
        		xtick={1,2,3,4,5,6,7,8,9,10,11},
                xticklabels={2,4,8,16,32,64,128,256,512,1024,2048},
        		x tick label style={rotate=45, anchor=north east, inner sep=0mm},
        		ylabel=$t$,
        		every axis y label/.style={rotate=0, black, at={(-0.13,0.5)},},
        	]
        		\addplot table[x=mu, y=trn100]{\averagehundredlog};
        		\addplot table[x=mu, y=t0n100]{\averagehundredlog};
        	\end{axis}
		\end{tikzpicture}
		}
        \caption{$n=100$}
        \label{fig:n100mulogscale}
    \end{subfigure}
    \caption{The average number of generations measured among $100$ runs at the time both optima were found on \twomax or $t=1$ million generations have been reached for $n=30$ and $n=100$, $\sigma=n/2$, $\kappa=1$ and $2\leq\mu\leq \kappa n^2/4$ for the populations with randomised and biased initialisation (blue and red line, respectively).}
    \label{fig:averunmu}
\end{figure*}
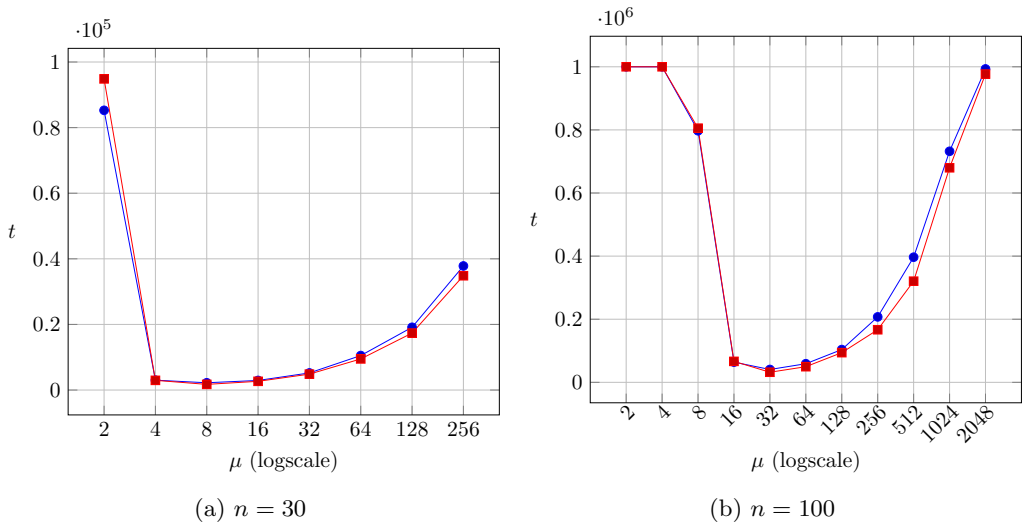

Most importantly, all population sizes in the interval $2\leq\mu\leq \kappa n^2/4$ are able to optimise \twomax; this experimental setting shows that with a smaller population size for a relatively small $n=30$ the algorithm is able to optimise \twomax. Another interesting characteristic of the algorithm is its capacity for escaping from a local optimum.

For $n=100$, in Figure~\ref{fig:n100mulogscale} it is more evident that with a small population size, it is not possible to escape from the basin of attraction of a peak, and the takeover happens before the population has the chance to evolve a distance of at least $n/2$ confirming the theoretical arguments described in Section~\ref{sec:smallpopdyn}. Once the population size is increased, the population is able to escape of the basin of attraction. The most interesting result shown in Figure~\ref{fig:averunmu} is that even with a population size $\mu\leq\kappa n^2/4$ the algorithm is able to find both optima on \twomax (even if runs require more than 1 million generations), indicating that the quadratic dependence on~$n$ in $\kappa n^2/4$, is an artefact of our approach. 

Finally, for larger $\mu$ sizes and $n=100$ it can be seen that biased initialisation is noticeably faster than random initialisation, and as the population grows the difference between the means grows. One reason could be simply because one peak has already been found, and the algorithm only needs to find the remaining peak.

\subsection{Escaping from Different Basins of Attraction}
\label{sec:escdifbasatt}

Finally, in this section we show that the runtime analysis used in Section~\ref{sec:pop-dyn}, and used to prove the theoretical analysis on the general classes of example landscapes functions in Section~\ref{sec:genexaland} can be used for weighted peak functions with two peaks of different Hamming distances. For simplicity we restrict our attention to equal slopes and heights: $a_1=a_2=1$ and $b_1=b_2=0$.

We can simplify this class of $f_2$ functions by using that the \muea is \emph{unbiased} as defined by \cite{Lehre2012}: simply speaking, the algorithm treats all bit values and all bit positions in the same way. Hence we can assume without loss of generality that $p_1=0^n$. We can further imagine shuffling all bits such that $p_2 = 0^{n-\Hamm{p_1}{p_2}}1^{\Hamm{p_1}{p_2}}$, which again does not change the stochastic behaviour of the \muea. Then all $f_2$ functions with $a_1=a_2=1$ and $b_1=b_2=0$ are covered by choosing $p_1=0^n$ and $p_2$ from the set $\{0^n,0^{n-1}1,0^{n-2}1^2,\ldots,1^n\}$. As can be seen the peaks $p_1$ and $p_2$ can be as close as $0^n$ and $0^{n-1}1$, or as far as $0^n$ and $1^n$. 

It contrast to the simple setting of \twomax where $\sigma=n/2$ makes most sense, in this more general setting it is necessary to define the \clearingradius $\sigma$ according to the Hamming distance between peaks. In particular, the following conditions should be satisfied.
\begin{enumerate}
\item $\sigma \le \Hamm{p_1}{p_2}$ as otherwise one peak is contained in the \clearingradius around the other peak,
\item $\sigma \le n/2$ as otherwise a niche can contain the majority of search points in the search space, leading to potentially exponential times to escape from the basin of attraction of a local optimum if $\sigma \ge (1+\Omega(1)) \cdot n/2$, and
\item $\sigma \ge \Hamm{p_1}{p_2}/2$ as this is the minimum distance that distinguishes the two peaks.
\end{enumerate}
In the following we study two different choices of $\sigma$: the maximum value $\sigma=\min\{\Hamm{p_1}{p_2},n/2\}$ that satisfied the above constraints, and the minimum feasible value, $\sigma=\Hamm{p_1}{p_2}/2$. These two choices allow us to investigate the effect of choosing large or small niches in this setting.

We use genotypic clearing with $\kappa=1$ and we make use of the results from Section~\ref{sec:empanpopsize} to define $\mu=32$ as a population size able to optimise \twomax for large $n=100$. We report the average number of generations of $100$ runs, with the same stopping criterion: both optima have been found or the algorithm has reached a maximum of 1 million generations. 

For the case of large \clearingradius, $\sigma=\min\{\Hamm{p_1}{p_2},n/2\}$, Figure~\ref{fig:hamming1} shows that it is possible to find both optima efficiently across the whole range of $\Hamm{p_1}{p_2}$. For random initialisation there are hardly any performance differences, except for a drop in the runtime when the two peaks get very close. 
For the case of biased populations we see differences by a small constant factor: the closer the peaks, the more difficult it is to escape (or find the new niche) since it requires to flip a specific number of bits to find the other optima. But as the two peaks move away both initialisation methods seem to behave the same indicating that the arguments used in Sections~\ref{sec:pop-dyn} and \ref{sec:genexaland} reflect correctly how the algorithm behaves.

Figure~\ref{fig:hamming2} shows that with the smallest feasible clearing radius $\sigma=\Hamm{p_1}{p_2}/2$ the algorithm is still able to find both optima for all $\Hamm{p_1}{p_2}$, but the average runtime for biased initialisation is much higher compared to $\sigma=\min\{\Hamm{p_1}{p_2},n/2\}$. From observing the actual population dynamics during a run, it seems that the reason for this high number of generations is  because several niches are created around both peaks, \ie once a peak has been found (and a niche is formed around the peak), the population spreads out by forming many niches between $p_1$ and $p_2$.

In the case of biased initialisation it is necessary to jump between specific niches to reach the opposite peak, or make several jumps between different niches in order to escape from its basin of attraction, which leads to this high number of generations.

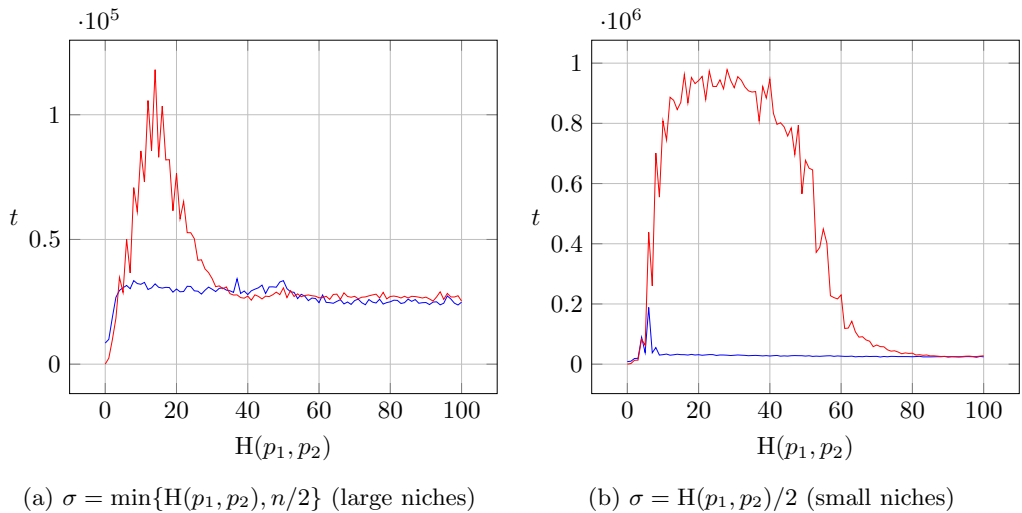
\begin{figure*}[ht]
    \centering
    \begin{subfigure}[t]{0.45\textwidth}
        \centering
        \resizebox{\linewidth}{!}{
        	\begin{tikzpicture}	
        	\begin{axis}[
        		width=6cm,
        		height=5cm,
        		xlabel=$\Hamm{p_1}{p_2}$,
        		ylabel=$t$,
        		every axis y label/.style={rotate=0, black, at={(-0.13,0.5)},},
        	]
        		\addplot[blue,smooth,tension=0] table[x=n, y=trn]{\hammingone};
                \addplot[red,smooth,tension=0] table[x=n, y=t0n]{\hammingone};
        	\end{axis}
		\end{tikzpicture}
		}
        \caption{$\sigma=\min\{\Hamm{p_1}{p_2},n/2\}$ (large niches)}
        \label{fig:hamming1}
    \end{subfigure}
    \begin{subfigure}[t]{0.45\textwidth}
        \centering
        \resizebox{\linewidth}{!}{
        	\begin{tikzpicture}	
        	\begin{axis}[
        		width=6cm,
        		height=5cm,
        		xlabel=$\Hamm{p_1}{p_2}$,
        		ylabel=$t$,
        		every axis y label/.style={rotate=0, black, at={(-0.13,0.5)},},
        	]
        		\addplot[blue,smooth,tension=0] table[x=n, y=trn]{\hammingtwo};
                \addplot[red,smooth,tension=0] table[x=n, y=t0n]{\hammingtwo};
        	\end{axis}
		\end{tikzpicture}
		}
        \caption{$\sigma=\Hamm{p_1}{p_2}/2$ (small niches)}
        \label{fig:hamming2}
    \end{subfigure}
    \caption{The average number of generations measured among $100$ runs at the time both peaks $p_1=0^n$ and $p_2=\{0^n,0^{n-1}1,0^{n-2}1^2,\ldots,1^n\}$ were found with $a_1=a_2=1$ and $b_1=b_2=0$ on the fitness landscape defined by $f_2$ or have reached $t=1$ million generations for $n=100$, with genotypic clearing with $\sigma=\min\{\Hamm{p_1}{p_2},n/2\}$ and $\sigma=\Hamm{p_1}{p_2}/2$, $\kappa=1$ and $\mu=32$ for populations with randomised (blue line) and biased (red line) initialisation.}
    \label{fig:hamming}
\end{figure*}

\section{Conclusions}
\label{sec:conclu}

The presented theoretical and empirical investigation has shown that \clearing possesses desirable and powerful characteristics. We have used rigorous theoretical analysis related to its ability to explore the landscape in two cases, small and large niches, and provide an insight into the behaviour of this diversity-preserving mechanism. 

In the case of small niches, we have proved that \clearing can exhaustively explore the landscape when the proper distance and parameters like \clearingradius, \nichecapacity and population size $\mu$ are set. Also, we have proved that \clearing is powerful enough to optimise all functions of unitation. In the case of large niches, \clearing has been proved to be as strong as other diversity-preserving mechanisms like \emph{deterministic crowding} and \emph{fitness sharing} since it is able to find both optima of the test function \twomax.

The analysis made has shown that our results can be easily extended to more general classes of examples landscapes. The analysis done for \twomax can easily be applied to different classes of bimodal problems using arguments based on how to escape the basin of attraction of one local optimum. We demonstrated this for functions with two complementary peaks and asymmetric variants of \twomax, consisting of a suboptimal peak with a smaller basin of attraction and an optimal peak with a larger basin of attraction. 

Our experimental results suggest that the same efficient performance also applies to bimodal functions where the two peaks have varying Hamming distances.
Here \clearing is able to escape from local optima with different basins of attractions by moving/jumping between niches formed by the \clearingradius. Defining $\sigma$ as the smallest possible value that allows to distinguish between peaks creates several small niches, forcing the individuals in the population to make several jumps between niches until an individual can reach the  basin of attraction of the other peak. This means that the algorithm requires more generations to find both peaks. But if $\sigma$ is defined as the maximum feasible value, $\sigma = \min\{\Hamm{p_1}{p_2}, n/2\}$, the \muea is faster and remarkably robust with respect to the Hamming distances between the two peaks. 
Nevertheless, both approaches allow the population to escape from different basin of attractions.

It remains an open problem to theoretically analyse the population dynamics of \clearing with more than 2 niches and to prove rigorously that \clearing is effective across a much broader range of problems, including problems with more than 2 peaks. This involves obtaining more detailed insights into the dynamics of the population, including the distribution and evolution of the losers across multiple niches. 

\section*{Acknowledgements}
The authors would like to thank Carsten Witt for his advice and comments on the analysis of the choice of the population size for \twomax, the anonymous reviewers of this manuscript and the previous FOGA publication for their many valuable suggestions. We also thank the Consejo Nacional de Ciencia y Tecnolog\'{i}a --- CONACYT (the Mexican National Council for Science and Technology) for the financial support under the grant no.\ 409151 and registration no.\ 264342. The research leading to these results has received funding from the European Union Seventh Framework Programme (FP7/2007-2013) under grant agreement no.\ 618091 (SAGE).

\bibliographystyle{apalike} 
\bibliography{example} 

\end{document}